\documentclass[11pt,oneside,letter]{article}
\usepackage[top=1 in, bottom=1 in, left=0.85 in, right=0.85 in]{geometry}
\usepackage{hyperref}

\usepackage{microtype}
\usepackage{graphicx}
\usepackage{subfigure}
\usepackage{enumitem}
\usepackage{booktabs} 
\usepackage{amsmath}
\usepackage{amsfonts}
\usepackage{multirow}
\usepackage[table]{xcolor}
\usepackage{xcolor}
\usepackage[noend,linesnumbered,ruled,vlined]{algorithm2e}

\SetCommentSty{mycommfont}

\usepackage{amssymb,amsthm}
\newtheorem{theorem}{Theorem}
\newtheorem{lemma}{Lemma}
\newtheorem{assumption}{Assumption}
\newtheorem{remark}{Remark}

\title{A Provably-Efficient Model-Free Algorithm for Constrained Markov Decision Processes}

\author{Honghao Wei, Xin Liu, and Lei Ying\\University of Michigan, Ann Arbor}
\date{}

\begin{document}
\maketitle

\begin{abstract}
This paper presents the first {\em model-free}, {\em simulator-free} reinforcement learning algorithm  for Constrained Markov Decision Processes (CMDPs) with sublinear regret and zero constraint violation. The algorithm is named {\em Triple-Q} because it includes three key components: a Q-function (also called action-value function) for the cumulative reward, a Q-function for the cumulative utility for the constraint, and a virtual-Queue that (over)-estimates the cumulative constraint violation. Under Triple-Q, at each  step, an action is chosen based on the pseudo-Q-value that is a combination of the three ``Q'' values. The algorithm updates the reward and utility Q-values with learning rates that depend on the visit counts to the corresponding (state, action) pairs and are periodically reset. In the episodic CMDP setting, Triple-Q achieves $\tilde{\cal O}\left(\frac{1 }{\delta}H^4 S^{\frac{1}{2}}A^{\frac{1}{2}}K^{\frac{4}{5}} \right)$ regret\footnote{{\bf Notation:} $f(n) = \tilde{\mathcal O}(g(n))$ denotes $f(n) = {\mathcal O}(g(n){\log}^k n)$ with $k>0.$ The same applies to $\tilde{\Omega}.$ $\mathbb R^+$ denotes non-negative real numbers. $[H]$ denotes the set $\{1,2,\cdots, H\}.$ }, where $K$ is the total number of episodes, $H$ is the number of steps in each episode, $S$ is the number of states, $A$ is the number of actions, and $\delta$ is Slater's constant. Furthermore, {Triple-Q} guarantees zero constraint violation, both on expectation and with a high probability, when $K$ is sufficiently large. Finally, the computational complexity of {Triple-Q} is similar to SARSA for unconstrained MDPs, and is computationally efficient.\end{abstract}

\section{Introduction}
Reinforcement learning (RL), with its success in gaming and robotics, has been widely viewed as one of the most important technologies for next-generation, AI-driven complex systems such as autonomous driving, digital healthcare, and smart cities. However, despite the significant advances  (such as deep RL) over the last few decades, a major obstacle in applying RL in practice is the lack of ``safety'' guarantees. 
Here ``safety'' refers to a broad range of operational constraints.  
The objective of a traditional RL problem is to maximize the expected cumulative reward, but in practice, many applications need to be operated under a variety of constraints, such as collision avoidance in robotics and autonomous driving \cite{OnoPavKuw_15,GarFer_12,FisAkaZei_18}, 
legal and business restrictions in financial engineering \cite{AbePreCez_10}, and resource and budget constraints in healthcare systems \cite{YuLiuNem_19}. 
These applications with operational constraints can often be modeled as  Constrained Markov Decision Processes (CMDPs), in which the agent's goal is to learn a policy that maximizes the expected cumulative reward subject to the constraints. 

Earlier studies on CMDPs assume the model is known. A comprehensive study of these early results can be found in \cite{Alt_99}. RL for unknown CMDPs has been a topic of great interest recently because of its importance in Artificial Intelligence (AI) and Machine Learning (ML). The most noticeable advances recently are {\em model-based} RL for CMDPs, where the transition kernels are learned and used to solve the linear programming (LP) problem for the CMDP \cite{SinGupShr_20,BraDudLyk_20,KriRahPie_20,EfrManPir_20}, or the LP problem in the primal component of a primal-dual algorithm \cite{QiuWeiYan_20,EfrManPir_20}. If the transition kernel is linear, then it can be learned in a sample efficient manner even for infinite state and action spaces, and then be used in the policy evaluation and improvement in a primal-dual algorithm \cite{DinWeiYan_20}. \cite{DinWeiYan_20} also proposes a model-based algorithm (Algorithm 3) for the tabular setting (without assuming a linear transition kernel). 
\begin{table*}[t]
	\caption{The Exploration-Exploitation Tradeoff in Episodic CMDPs.}
	\label{ta:algorithms}
	\begin{center}
		\begin{tabular}{|c|l|l|l|}
			\toprule
			& {\bf Algorithm} & {\bf Regret} & {\bf Constraint Violation} \\
			\hline
			 \multirow{3}{*}{Model-based} & OPDOP \cite{DinWeiYan_20} & $\tilde{\mathcal{O}}(H^3\sqrt{S^2AK})$& $\tilde{\mathcal{O}}(H^3\sqrt{S^2AK})$ \\\cline{2-4}
				&  OptDual-CMDP \cite{EfrManPir_20}  &	$\tilde{\mathcal{O}}(H^2\sqrt{S^3AK})$ & $\tilde{\mathcal{O}}(H^2\sqrt{S^3AK})$  \\\cline{2-4}
				&  OptPrimalDual-CMDP \cite{EfrManPir_20}  &	$\tilde{\mathcal{O}}(H^2\sqrt{S^3AK})$ & $\tilde{\mathcal{O}}(H^2\sqrt{S^3AK})$  \\ 
		\bottomrule
			\toprule
		Model-free & \cellcolor{lightgray} Triple-Q &$\tilde{\cal O}\left(\frac{1 }{\delta}H^4 S^{\frac{1}{2}}A^{\frac{1}{2}}K^{\frac{4}{5}} \right)$ &0 \\
			\bottomrule
		\end{tabular}
	\end{center}
\end{table*}

The performance of a model-based RL algorithm depends on how accurately a model can be estimated. For some complex environments, building accurate models is challenging computationally and data-wise \cite{SutBar_18}. For such environments, model-free RL algorithms often are more desirable. However, there has been little development on model-free RL algorithms for CMDPs with provable optimality or regret guarantees, with the exceptions \cite{DinZhaBas_20,XuLiaLan_20,CheDonWan_21}, all of which require simulators. In particular, the sample-based NPG-PD algorithm in \cite{DinZhaBas_20} requires a simulator which can simulate the MDP from any initial state $x,$ and the algorithms in \cite{XuLiaLan_20,CheDonWan_21} both require a simulator for policy evaluation. It has been argued in \cite{AzaRemHil_12,AzaRemHil_13,JinAllZey_18} that with a perfect simulator, exploration is not needed and sample efficiency can be easily achieved because the agent can query any (state, action) pair as it wishes. Unfortunately, for complex environments, building a perfect simulator often is as difficult as deriving the model for the CMDP. For those environments, sample efficiency and the  exploration-exploitation tradeoff are critical and become one of the most important considerations of RL algorithm design. 
\subsection{Main Contributions}
In this paper, we consider the online learning problem of an episodic CMDP with a model-free approach {\em without} a simulator. We develop the first {\em model-free} RL algorithm for CMDPs with sublinear regret and {\em zero} constraint violation (for large $K$). 
The algorithm is named {Triple-Q} because it has three key components: (i) a Q-function (also called action-value function) for the expected cumulative reward, denoted by $Q_{h}(x,a)$ where $h$ is the step index and $(x,a)$ denotes a state-action pair, (ii) a Q-function for the expected cumulative utility for the constraint, denoted by $C_{h}(x,a),$ and (iii) a virtual-Queue, denoted by $Z,$ which (over)estimates the cumulative constraint violation so far. At  step $h$ in the current episode, when observing state $x,$ the agent selects action $a^*$ based on a {\em pseudo-Q-value} that is a combination of the three ``Q'' values:
\begin{align*}
a^*\in \underbrace{\arg\max_a  Q_{h}(x,a)+\frac{Z}{\eta} C_{h}(x,a)}_{{\text{pseudo-Q-value of state $(x,a)$ at step $h$}}},    
\end{align*}  where $\eta$ is a constant. Triple-Q uses UCB-exploration when learning the Q-values, where the UCB bonus and the learning rate at each update both depend on the visit count to the corresponding (state, action) pair as in \cite{JinAllZey_18}). Different from the optimistic Q-learning for unconstrained MDPs (e.g. \cite{JinAllZey_18,WanDonChe_20,WeiJahLuo_20}), the learning rates in Triple-Q need to be periodically reset at the beginning of each frame, where a frame consists of $K^\alpha$ consecutive episodes. The value of the virtual-Queue (the dual variable) is updated once in every frame. So Triple-Q can be viewed as a two-time-scale algorithm where virtual-Queue is updated at a slow time-scale, and Triple-Q learns the pseudo-Q-value for fixed $Z$ at a fast time scale within each frame. Furthermore, it is critical to update the two Q-functions ($Q_h(x,a)$ and $C_h(x,a)$) following a rule similar to SARSA \cite{RumNir_94} instead of Q-learning \cite{Wat_89}, in other words, using the Q-functions of the action that is taken instead of using the $\max$ function.    

We prove {Triple-Q} achieves $\tilde{\cal O}\left(\frac{1 }{\delta}H^4 S^{\frac{1}{2}}A^{\frac{1}{2}}K^{\frac{4}{5}} \right)$ reward regret and guarantees {\em zero} constraint violation when the total number of episodes $K\geq\left(\frac{16\sqrt{SAH^6\iota^3}}{\delta}\right)^5,$ where $\iota$ is logarithmic in $K.$
Therefore, in terms of constraint violation, our bound is sharp for large $K.$ To the best of our knowledge, this is the first {\em model-free}, {\em simulator-free} RL algorithm with sublinear regret and {\em zero} constraint violation. For model-based approaches, it has been shown  that a model-based algorithm achieves both $\tilde{\cal O}(\sqrt{H^4SAK})$ regret and constraint violation (see, e.g. \cite{EfrManPir_20}). It remains open what is the fundamental lower bound on the regret under model-free algorithms for CMDPs and whether the regret bound under Triple-Q is order-wise sharp or can be further improved. Table \ref{ta:algorithms} summarizes the key results on the exploration-exploitation tradeoff of CMDPs in the literature.

As many other model-free RL algorithms, a major advantage of {Triple-Q} is its low computational complexity.  The computational complexity of {Triple-Q} is similar to SARSA for unconstrained MDPs, so it retains both its effectiveness and efficiency  while solving a much harder problem. While we consider a tabular setting in this paper, {Triple-Q} can  easily incorporate function approximations (linear  function approximations or neural networks) by replacing the $Q(x,a)$ and $C(x,a)$ with their function approximation versions, making the algorithm a very appealing approach for solving complex CMDPs in practice. We will demonstrate this by applying Deep Triple-Q, Triple-Q with neural networks, to the Dynamic Gym benchmark \cite{YanSimTin_21} in Section \ref{sec:simu}.

\section{Problem Formulation}
We consider an episodic CMDP, denoted by $(\mathcal{S},\mathcal{A},H,\mathbb{P},r,g),$ where $\mathcal{S}$ is the state space with $\vert \mathcal{S}\vert=S,$ $\mathcal{A}$ is the action space with $\vert \mathcal{A}\vert=A,$ $H$ is the number of steps in each episode, and $\mathbb{P}=\{\mathbb{P}_h\}_{h=1}^H$ is a collection of transition kernels (transition probability matrices). At the beginning of each episode, an initial state $x_{1}$ is sampled from distribution $\mu_0.$ Then at step $h,$ the agent takes action $a_h$ after observing state $x_h$. Then the agent receives a reward $r_h(x_h,a_h)$ and incurs a utility $g_h(x_h,a_h).$ The environment then moves to a new state $x_{h+1}$ sampled from distribution $\mathbb{P}_h(\cdot\vert x_h,a_h).$ Similar to \cite{JinAllZey_18}, we assume that  $r_h(x,a)(g_h(x,a)):\mathcal{S}\times \mathcal{A}\rightarrow [0,1],$ are deterministic for convenience.

Given a policy $\pi,$ which is a collection of $H$ functions $\{\pi_h:\mathcal{S}\rightarrow\mathcal{A}\}_{h=1}^H,$ the reward value function $V_{h}^\pi$ at step $h$ is the expected cumulative rewards from step $h$ to the end of the episode under policy $\pi:$ 
$$V_{h}^\pi(x)=\mathbb{E}\left[\left.\sum_{i=h}^H r_i(x_i,\pi_i(x_i))\right\vert x_h=x \right].$$ 
The (reward) $Q$-function $Q_{h}^\pi(x,a)$ at step $h$ is the expected cumulative rewards when agent starts from a state-action pair $(x,a)$ at step $h$ and then follows policy $\pi:$ 
$$Q_{h}^\pi(x,a  )=r_h(x,a) + \mathbb{E}\left[\left.\sum_{i=h+1}^H r_i(x_i,\pi_i(x_i))\right\vert\begin{array}{cc}
      x_h=x\\a_h=a  
\end{array} \right].$$ Similarly, we use $W_{h}^\pi(x):\mathcal{S}\rightarrow\mathbb{R}^+$ and $C_{h}^\pi(x,a):\mathcal{S}\times\mathcal{A}\rightarrow\mathbb{R}^+$ to denote the utility value function and utility $Q$-function at step $h$:
\begin{align*}
	W_{h}^\pi(x) &=\mathbb{E}\left[\left.\sum_{i=h}^H g_i(x_i,\pi_i(x_i))\right\vert x_h=x \right],\\
	C_{h}^\pi(x,a) &=g_h(x,a) + \mathbb{E}\left[\left.\sum_{i=h+1}^H g_i(x_i,\pi_i(x_i))\right\vert \begin{array}{cc}
	     x_h=x\\a_h=a 
	\end{array}\right].
\end{align*}

For simplicity, we adopt the following notation (some used in \cite{JinAllZey_18,DinWeiYan_20}): 
\begin{align*}
	\mathbb{P}_hV_{h+1}^\pi(x,a)=\mathbb{E}_{x^\prime\sim\mathbb{P}_h(\cdot\vert x,a)}V^\pi_{h+1}(x^\prime), &\quad 	Q_{h}^\pi(x,\pi_h(x)) = \sum_a 	Q_{h}^\pi(x,a) \mathbb{P}(\pi_h(x)=a)\\	
	\mathbb{P}_hW_{h+1}^\pi(x,a)=\mathbb{E}_{x^\prime\sim\mathbb{P}_h(\cdot\vert x,a)}W^\pi_{h+1}(x^\prime), &\quad 
	C_{h}^\pi(x,\pi_h(x)) = \sum_a 	C_{h}^\pi(x,a) \mathbb{P}(\pi_h(x)=a).
\end{align*} From the definitions above, we have  
\begin{align*}
V_{h}^\pi(x)=Q_{h}^\pi(x,\pi_h(x)),\quad &Q_{h}^\pi(x,a)=r_h(x,a)+\mathbb{P}_h V_{h+1}^\pi(x,a),
\end{align*} and
\begin{align*}
W_{h}^\pi(x)=C_{h}^\pi(x,\pi_h(x)),\quad &C_{h}^\pi(x,a)=g_h(x,a)+\mathbb{P}_h W_{h+1}^\pi(x,a).
\end{align*}

Given the model defined above, the objective of the agent is to find a policy that maximizes the expected cumulative reward subject to a constraint on the expected utility:
\begin{equation}
    \underset{\pi\in\Pi}{\text{maximize}}\  \mathbb E\left[V_{1}^\pi(x_1)\right] \  \text{subject to:} \mathbb E\left[W_{1}^\pi(x_1)\right]\geq \rho,\label{eq:obj}
\end{equation}
where we assume $\rho\in[0,H]$ to avoid triviality and the expectation is taken with respect to the initial distribution $x_1\sim \mu_0.$ 

\begin{remark}
The results in the paper can be directly applied to a constraint in the form of  
\begin{equation}
\mathbb E\left[W_{1}^\pi(x_1)\right]\leq \rho.\label{cost-constraint}
\end{equation} Without loss of generality, assume $\rho\leq H.$ We define $\tilde{g}_h(x,a)=1-g_h(x,a)\in[0,1]$ and $\tilde{\rho}=H-\rho\geq 0,$ the the constraint in \eqref{cost-constraint} can be written as 
\begin{equation}
\mathbb E\left[\tilde{W}_{1}^\pi(x_1)\right]\geq \tilde{\rho},
\end{equation} where 
\begin{align*}
\mathbb E\left[\tilde{W}_{1}^\pi(x_1)\right]=\mathbb{E}\left[\sum_{i=1}^H \tilde{g}_i(x_i,\pi_i(x_i)) \right]=H-\mathbb E\left[W_{1}^\pi(x_1)\right]. 
\end{align*}
\end{remark}

Let $\pi^*$ denote the optimal solution to the CMDP problem defined in \eqref{eq:obj}. We evaluate our model-free RL algorithm using regret and constraint violation defined below:

\begin{align}
	\text{Regert}(K) &=\mathbb E\left[\sum_{k=1}^K\left(V_{1}^*(x_{k,1})-V_{1}^{\pi_k}(x_{k,1})\right)\right],\label{def:regret}\\
	\text{Violation}(K) &=\mathbb E\left[ \sum_{k=1}^K \left(\rho-W_1^{\pi_k}(x_{k,1})\right)\right],
\end{align}
where $V_1^*(x)=V_1^{\pi^*}(x),$ $\pi_k$ is the policy used in episode $k$ 
and the expectation is taken with respect to the distribution of the initial state $x_{k,1} \sim  \mu_0.$

In this paper, we assume the following standard Slater's condition hold.
\begin{assumption}
	\label{as:1}
	(Slater's Condition). Given initial distribution  $\mu_0,$ there exist $\delta>0$ and policy $\pi$ such that $$\mathbb E\left[W_{1}^\pi(x_1)\right] -\rho \geq \delta.$$
\end{assumption}
In this paper, Slater's condition simply means there exists a feasible policy that can satisfy the constraint with a slackness $\delta.$  This has been commonly used in the literature \cite{DinWeiYan_20,DinZhaBas_20,EfrManPir_20,PatChaCal_19}. We call $\delta$ Slater's constant. While the regret and constraint violation bounds depend on $\delta,$ our algorithm does not need to know $\delta$ under the assumption that $K$ is large (the exact condition can be found in Theorem \ref{thm:main}). This is a noticeable difference from some of works in CMDPs in which the agent needs to know the value of this constant (e.g. \cite{DinWeiYan_20}) or alternatively a feasible policy (e.g. \cite{AchHelDav_17}) .

\section{{Triple-Q}}
In this section, we introduce {Triple-Q} for CMDPs. The design of our algorithm is based on the primal-dual approach in optimization. While RL algorithms based on the primal-dual approach have been developed for CMDPs \cite{DinWeiYan_20,DinZhaBas_20,QiuWeiYan_20,EfrManPir_20}, a model-free RL algorithm with sublinear regrets and {\em zero} constraint violation is new.  

The design of {Triple-Q} is based on the primal-dual approach in optimization. Given Lagrange multiplier $\lambda,$ we consider the Lagrangian of problem \eqref{eq:obj} from a given initial state $x_1:$
\begin{align}
&\max_\pi V_{1}^\pi(x_1)+ \lambda\left(W_{1}^\pi(x_1)-\rho\right)\label{eq:lag}\\
=& \max_\pi \mathbb{E}\left[\sum_{h=1}^H r_h(x_h,\pi_h(x_h))+\lambda g_h(x_h,\pi_h(x_h)) \right]-\lambda\rho,\nonumber
\end{align} which is an unconstrained MDP with reward $r_h(x_h,\pi_h(x_h))+\lambda g_h(x_h,\pi_h(x_h))$ at step $h.$ Assuming we solve the unconstrained MDP and obtain the optimal policy, denoted by $\pi^*_\lambda,$ we can then update the dual variable (the Lagrange multiplier) using a gradient method:
$$\lambda \leftarrow \left(\lambda +\rho - \mathbb E\left[W_1^{\pi_\lambda^*}(x_1)\right]\right)^+.$$ While  primal-dual is a standard approach, analyzing the finite-time performance such as regret or sample complexity is particularly challenging. For example, over a finite learning horizon, we will not be able to exactly solve the unconstrained MDP for given $\lambda.$ Therefore, we need to carefully design how often the Lagrange multiplier should be updated. If we update it too often, then the algorithm may not have sufficient time to solve the unconstrained MDP, which leads to divergence; and on the other hand, if we update it too slowly, then the solution will converge slowly to the optimal solution and will lead to large regret and constraint violation. Another challenge is that when $\lambda$ is given, the primal-dual algorithm solves a problem with an objective different from the original objective and does not consider any constraint violation. Therefore, even when the asymptotic convergence may be established, establishing the finite-time regret is still difficult because we need to evaluate the difference between the policy used at each step and the optimal policy.

Next we will show that a low-complexity primal-dual algorithm can converge and have sublinear regret and {\em zero} constraint violation when carefully designed. In particular, {Triple-Q} includes the following key ideas:  
\begin{itemize}[leftmargin=*]
    \item 
    A sub-gradient algorithm for estimating the Lagrange multiplier, which is updated at the beginning of each frame as follows: 
\begin{align}
	Z \leftarrow \left(Z+ \rho+\epsilon-  \frac{\bar{C}}{K^\alpha}\right)^+, \label{eq:vq}
\end{align} where  $(x)^+=\max\{x,0\}$ and $\bar{C}$ is the summation of all $C_{1}(x_{1},a_{1})$s of the episodes in the previous frame. We call $Z$ a virtual queue because it is terminology that has been widely used in stochastic networks  (see e.g. \cite{Nee_10,SriYin_14}). If we view $\rho+\epsilon$ as the number of jobs that arrive at a queue within each frame and $\bar{C}$ as the number of jobs that leave the queue within each frame, then $Z$ is the number of jobs that are waiting at the queue.
Note that we added extra utility $\epsilon$ to $\rho.$ By choosing $\epsilon =\frac{8\sqrt{SAH^6\iota^3}}{K^{0.2}}$, the virtual queue pessimistically estimates constraint violation so Triple-Q achieves {\em zero} constraint violation when the number of episodes is large.  

\item A carefully chosen parameter $\eta=K^{0.2}$ so that when $\frac{Z}{\eta}$ is used as the estimated Lagrange multiplier, it balances the trade-off between maximizing the cumulative reward and satisfying the constraint.

\item Carefully chosen learning rate $\alpha_t$ and Upper Confidence Bound (UCB) bonus $b_t$ to guarantee that the estimated Q-value does not significantly deviate from the actual Q-value. We remark that the learning rate and UCB bonus proposed for unconstrained MDPs \cite{JinAllZey_18} do not work here. Our learning rate is chosen to be $\frac{K^{0.2}+1}{K^{0.2}+t},$ where $t$ is the number of visits to a given (state, action) pair in a particular step. This decays much slower than the classic learning rate $\frac{1}{t}$ or $\frac{H+1}{H+t}$ used in \cite{JinAllZey_18}. The learning rate is further reset from frame to frame, so Triple-Q can continue to learn the pseudo-Q-values that vary from frame to frame due to the change of the virtual-Queue (the Lagrange multiplier).  
\end{itemize}

We now formally introduce {Triple-Q}.  A detailed description is presented in Algorithm \ref{alg:triple-Q}.  The algorithm only needs to know the values of $H,$ $A,$ $S$ and $K,$ and no other problem-specific values are needed.  Furthermore, Triple-Q includes updates of two Q-functions per step: one for $Q_{h}$ and one for $C_{h};$ and one simple virtual queue update per frame. So its computational complexity is similar to SARSA. 

\begin{algorithm}[htb]
	\label{alg:triple-Q}
	\SetAlgoLined
	Choose $\chi = K^{0.2},$  $\eta= K^{0.2},$ $\iota=128\log\left(\sqrt{2SAH}K\right),\alpha=0.6$ and  $\epsilon = \frac{8\sqrt{SAH^6\iota^3}}{K^{0.2}}$ \;
	Initialize  $Q_{h}(x,a)=C_{h}(x,a)\leftarrow H$ and  $Z=\bar{C}=N_{h}(x,a)=V_{H+1}(x)=W_{H+1}(x)\leftarrow 0$ for all $(x,a,h)\in{\cal S}\times {\cal A}\times [H]$\;
	\For{episode $k = 1,\dots,K $ }{
		Sample the initial state for episode $k:$ $x_{1} \sim \mu_0$\;
		\For{step $h=1,\dots,H+1$}{
			\If{$h\leq H$ \tcp*{\small take a greedy action based on the pseudo-Q-function}}{
				Take action $a_{h} \leftarrow \arg\max_a\left( Q_{h} (x_{h},a) + \frac{Z}{\eta}C_{h}(x_{h},a)\right)$\;
				Observe $r_h(x_{h},a_{h}), g_h(x_{h},a_{h}), $ and $x_{h+1}$ \;
				$N_{h}(x_{h},a_{h})\leftarrow N_{h}(x_{h},a_{h})+1,V_{h}(x_{h})\leftarrow Q_{h}(x_{h},a_{h}), W_{h}(x_{h})\leftarrow C_{h}(x_{h},a_{h})$;
			}
			\If{$h\geq 2$ \tcp*{\small update the Q-values for $(x_{h-1},a_{h-1})$ after observing $(s_h,a_h)$}}{
				Set $t=N_{h-1}(x_{h-1},a_{h-1}), b_t = \frac{1}{4}\sqrt{\frac{H^2\iota\left(\chi +1 \right) }{\chi+t}}, \alpha_t=\frac{\chi+1}{\chi+t}$ \;
				Update the reward Q-value: ${Q}_{h-1} (x_{h-1},a_{h-1})\leftarrow (1-\alpha_t)Q_{h-1} (x_{h-1},a_{h-1}) + \alpha_t\left(r_{h-1}(x_{h-1},a_{h-1})+V_{h} (x_{h})+b_t\right)$\;
				Update the utility Q-value: ${C}_{h-1} (x_{h-1},a_{h-1})\leftarrow (1-\alpha_t)C_{h-1} (x_{h-1},a_{h-1}) + \alpha_t\left(g_{h-1}(x_{h-1},a_{h-1})+W_{h}(x_{h})+b_t\right)$\;
			}
			\If{$h=1$}{
				$\bar{C} \leftarrow \bar{C}+C_{1}(x_{1},a_{1})$ \tcp*{\small Add $C_1(x_1,a_1)$ to $\bar{C}$} 	
			}
		}
		\If{$k\mod(K^\alpha)=0$  \tcp*{\small Reset the visit counts, add extra bonuses, and update the virtual-queue at the beginning of each frame}} { 
			$N_{h}(x,a)\leftarrow 0, Q_{h}(x,a)\leftarrow Q_{h}(x,a)+\frac{2H^3\sqrt{\iota}}{\eta}, \forall (x,a,h)$\;
			\If{$Q_{h}(x,a)\geq H$ or $C_{h}(x,a)\geq H$}{
				$Q_{h}(x,a)\leftarrow H$ and $C_{h}(x,a)\leftarrow H;$
			}
			$Z\leftarrow \left(Z +\rho+\epsilon -  \frac{\bar{C}}{K^\alpha}\right)^+,$ and   $\bar{C}\leftarrow 0$ \tcp*{\small  update the virtual-queue length}
		}
	}
	\caption{{Triple-Q}}
\end{algorithm}

The next theorem summarizes the regret and constraint violation bounds guaranteed under {Triple-Q}.

\begin{theorem}\label{thm:main}
Assume $K \geq \left(\frac{16\sqrt{SAH^6\iota^3}}{\delta} \right)^5,$ where $\iota=128\log(\sqrt{2SAH}K).$ Triple-Q achieves the following regret and constraint violation bounds:  
\begin{align*}
	\text{\em Regret}(K) & \leq \frac{13}{\delta} H^4 {\sqrt{SA\iota^3}}{K^{0.8}}  +\frac{4H^4\iota}{ K^{1.2}} \\
	\text{\em Violation}(K) &\leq   \frac{54H^4{\iota}K^{0.6} }{\delta}\log{\frac{16H^2\sqrt{\iota}}{\delta}} 
	+\frac{4\sqrt{H^2\iota}}{\delta}K^{0.8}- 5\sqrt{SAH^6\iota^3}K^{0.8}.
		\end{align*}
If we further have $K\geq e^{\frac{1}{\delta}},$ then $\text{\em Violation}(K)\leq 0$ and
\begin{align*}
\Pr\left(\sum^{K}_{k=1}\rho-W_{1}^{\pi_k} (x_{k,1})\leq 0\right)= 1-\tilde{\mathcal{O}}\left(e^{-K^{0.2} + }+\frac{1}{K^2} \right),
\end{align*}  in other words, Triple-Q guarantees zero constraint violation both on expectation and with a high probability. 
\end{theorem}

 \section{Proof of the Main Theorem}
We now present the complete proof of the main theorem. 
\subsection{Notation}
In the proof, we explicitly include the episode index in our notation. In particular, \begin{itemize}
    \item $x_{k,h}$ and  $a_{k,h}$ are the state and the action taken at step $h$ of episode $k.$
    
    \item $Q_{k,h},$ $C_{k,h},$ $Z_k,$ and $\bar C_k$ are the reward Q-function, the utility Q-function, the virtual-Queue, and the value of $\bar C$ {\em at the beginning} of episode $k.$ 
    
    \item $N_{k,h},$ $V_{k,h}$ and $W_{k,h}$ are the visit count, reward value-function, and utility value-function {\em after} they are updated at step $h$ of episode $k$ (i.e. after line 9 of Triple-Q). 
    \end{itemize} 
We also use shorthand notation $$\{f-g\}(x)=f(x)-g(x),$$ when $f(\cdot)$ and $g(\cdot)$ take the same argument value. Similarly $$\{(f-g)q\}(x)=(f(x)-g(x))q(x).$$ In this shorthand notation, we put functions inside $\{\ \},$ and the common argument(s) outside. 

A summary of notations used throughout this paper can be found in Table \ref{ta:notations} in the appendix. 
\subsection{Regret}
To bound the regret, we consider the following offline optimization problem as our regret baseline  \cite{Alt_99,PutMar_14}:
\begin{align}
	\underset{q_h }{\max} &\sum_{h,x,a} q_h(x,a)r_h(x,a) \label{eq:lp}\\
	\hbox{s.t.:} &\sum_{h,x,a} q_h(x,a)g_h(x,a) \geq \rho \label{lp:cost}\\
	& \sum_a q_h(x,a) = \sum_{x^\prime,a^\prime} \mathbb P_{h-1}(x\vert x^\prime,a^\prime) q_{h-1} (x^\prime,a^\prime)\label{lp:gb}\\
	& \sum_{x,a}q_h(x,a) =1, \forall h\in[H] \label{lp:normalization}\\
	&\sum_{a} q_1(x,a)= \mu_0(x)\label{lp:ini} \\
	& q_h(x,a) \geq 0, \forall x\in\mathcal{S},\forall a\in\mathcal{A},\forall h\in[H]. \label{lp:p}
\end{align}	
Recall that $\mathbb P_{h-1}(x\vert x^\prime,a^\prime)$ is the probability of transitioning to state $x$ upon taking action $a'$ in state $x'$ at step $h-1.$  This optimization problem is linear programming (LP), where $q_h(x,a)$ is the probability of (state, action) pair $(x,a)$ occurs in step $h,$ $\sum_a q_h(x,a)$ is the probability the environment is in state $x$ in step $h,$ and $$\frac{q_h(x,a)}{\sum_{a'} q_h(x,a')}$$ is the probability of taking action $a$ in state $x$ at step $h,$ which defines the policy. We can see that \eqref{lp:cost} is the utility constraint, \eqref{lp:gb} is the global-balance equation for the MDP, \eqref{lp:normalization} is the normalization condition so that $q_h$ is a valid probability distribution, and \eqref{lp:ini} states that the initial state is sampled from $\mu_0.$ Therefore, the optimal solution to this LP solves the CMDP (if the model is known), so we use the optimal solution to this LP as our baseline.  

To analyze the performance of {Triple-Q}, we need to consider a tightened version of the LP, which is defined below: 
\begin{align}
	\underset{q_h }{\max} &\sum_{h,x,a} q_h (x,a)r_h(x,a) \label{eq:lp-epsilon}\\
	\hbox{s.t.:} &\sum_{h,x,a} q_h (x,a)g_h(x,a)  \geq \rho+\epsilon \nonumber\\
 &\eqref{lp:gb}-\eqref{lp:p}\nonumber,
 \end{align} where $\epsilon>0$ is called a tightness constant. When $\epsilon\leq \delta,$ this problem has a feasible solution due to Slater's condition. We use superscript ${ }^*$ to denote the optimal value/policy related to the original CMDP \eqref{eq:obj} or the solution to the corresponding LP \eqref{eq:lp} and superscript ${ }^{\epsilon,*}$ to denote the optimal value/policy related to the $\epsilon$-tightened version of CMDP (defined in  \eqref{eq:lp-epsilon}). 
 
 Following the definition of the regret in \eqref{def:regret}, we have
\begin{align}
    \hbox{Regret}(K) = \mathbb E\left[\sum_{k=1}^{K}  V^*_{1}(x_{k,1}) -  V_{1}^{\pi_k}(x_{k,1}) \right] =\mathbb E\left[\sum_{k=1}^{K} \left( \sum_a \left\{Q^*_{1}q_1^*\right\}(x_{k,1},a)\right) -  Q_{1}^{\pi_k}(x_{k,1}, a_{k,1}) \right] \nonumber.
\end{align}
Now by adding and subtracting the corresponding terms, we obtain 
\begin{align}
    &  \hbox{Regret}(K)\nonumber\\
    = & \mathbb E \left[\sum_{k=1}^{K} \left( \sum_a  \left\{{Q}_{1}^{*}{q}^{*}_1 -{Q}_{1}^{\epsilon,*}{q}^{\epsilon,*}_1\right\}(x_{k,1},a)    \right)\right]  +\label{step:epsilon-dif} \\
     &\mathbb E \left[\sum_{k=1}^{K}  \left(  \sum_a \left\{{Q}_{1}^{\epsilon,*}{q}^{\epsilon,*}_1\right\}(x_{k,1},a)-Q_{k,1}(x_{k,1}, a_{k,1}) \right)\right]+\label{step(i)}\\
    &\mathbb E \left[\sum_{k=1}^{K}  \left\{Q_{k,1}-  Q_{1}^{\pi_k}\right\}(x_{k,1}, a_{k,1}) \right].\label{step:biase}
\end{align}

Next, we establish the regret bound by analyzing the three terms above. We first present a brief outline. 

\subsubsection{Outline of the Regret Analysis}
\begin{itemize}[leftmargin=*]
\item {\bf Step 1:} First, by comparing the LP associated with the original CMDP \eqref{eq:lp} and the tightened LP \eqref{eq:lp-epsilon}, Lemma \ref{le:epsilon-dif} will show  
	$$\mathbb{E}\left[ \sum_a \left\{Q_{1}^{*}q_1^*-  {Q}_{1}^{\epsilon,*}{q}_1^{\epsilon,*}\right\}(x_{k,1},a)\right] \leq \frac{H \epsilon }{\delta},$$  which implies that under our choices of $\epsilon,$ $\delta,$ and $\iota,$
$$\eqref{step:epsilon-dif} \leq \frac{KH\epsilon}{\delta}= \tilde{\cal O}\left(\frac{1 }{\delta}H^4 S^{\frac{1}{2}}A^{\frac{1}{2}}K^{\frac{4}{5}} \right).$$ 

\item {\bf Step 2:} Note that $Q_{k,h}$ is an estimate of $Q^{\pi_k}_h,$ and the estimation error \eqref{step:biase} is controlled by the learning rates and the UCB bonuses. In Lemma  \ref{le:qk-qpi-bound}, we will show that the cumulative estimation error over one frame is  upper bounded by $$H^2SA  +\frac{H^3\sqrt{\iota}K^\alpha}{\chi} +  \sqrt{H^4SA\iota K^{\alpha}(\chi+1)}.$$
Therefore,  under our choices of $\alpha,$ $\chi,$ and $\iota,$	the cumulative estimation error over $K$ episodes satisfies 
\begin{align*}
\eqref{step:biase}\leq H^2SA K^{1-\alpha}  +\frac{H^3\sqrt{\iota}K}{\chi} +  \sqrt{H^4SA\iota K^{2-\alpha}(\chi+1)}=\tilde{\cal O}\left(H^3 S^{\frac{1}{2}}A^{\frac{1}{2}}K^{\frac{4}{5}} \right).
\end{align*} The proof of Lemma \ref{le:qk-qpi-bound} is based on a recursive formula that relates the estimation error at step $h$ to the estimation error at step $h+1,$ similar to the one used in \cite{JinAllZey_18}, but with different learning rates and UCB bonuses. 

\item {\bf Step 3:} Bounding \eqref{step(i)} is the most challenging part of the proof. For unconstrained MDPs, the optimistic Q-learning in \cite{JinAllZey_18} guarantees that $Q_{k,h}(x,a)$ is an overestimate of $Q^*_h(x,a)$ (so also an overestimate of $Q^{\epsilon,*}_h(x,a)$) for all $(x,a,h,k)$ simultaneously with a high probability. However, this result does not hold under Triple-Q because Triple-Q takes greedy actions with respect to the pseudo-Q-function instead of the reward Q-function. To overcome this challenge, we first add and subtract additional terms to obtain
\begin{align}
  &\mathbb E \left[\sum_{k=1}^{K}  \left(  \sum_a \left\{{Q}_{1}^{\epsilon,*}{q}^{\epsilon,*}_1\right\}(x_{k,1},a)-Q_{k,1}(x_{k,1}, a_{k,1}) \right)\right]\nonumber\\
    =&
    \mathbb{E}\left[ \sum_{k} \sum_a \left(\left\{{Q}_{1}^{\epsilon,*}{q}^{\epsilon,*}_1+\frac{Z_k}{\eta} C_{1}^{\epsilon,*}{q}^{\epsilon,*}_1\right\}(x_{k,1},a) -  \left\{Q_{k,1}{q}^{\epsilon,*}_1 +\frac{Z_k}{\eta}C_{k,1}{q}^{\epsilon,*}_1\right\}(x_{k,1},a)\right)\right] \label{F-new}\\
	&+ \mathbb{E}\left[\sum_{k} \left(\sum_a  \left\{Q_{k,1} {q}^{\epsilon,*}_1\right\}(x_{k,1},a)	- Q_{k,1} (x_{k,1},a_{k,1})\right)\right] +\mathbb{E}\left[ \sum_{k} \frac{Z_k}{\eta} \sum_a \left\{\left(C_{k,1} - {C}^{\epsilon,*}_1   \right){q}^{\epsilon,*}_1\right\}(x_{k,1},a) \right].\label{eq:(i)expanded-new}
\end{align} We can see \eqref{F-new} is the difference of two pseudo-Q-functions. Using a three-dimensional induction (on step, episode, and frame), we will prove in Lemma \ref{le:qk-qpi-relation} that  $\left\{Q_{k,h} +\frac{Z_k}{\eta}C_{k,h}\right\}(x,a)$ is an overestimate of $\left\{{Q}_{h}^{\epsilon,*}+\frac{Z_k}{\eta} C_{h}^{\epsilon,*}\right\}(x,a)$ (i.e. $\eqref{F-new}\leq 0$) for all $(x,a,h,k)$ simultaneously with a high probability. Since $Z_k$ changes from frame to frame, Triple-Q adds the extra bonus in line 21 so that the induction can be carried out over frames. 

Finally, to bound \eqref{eq:(i)expanded-new}, we use the Lyapunov-drift method and consider Lyapunov function $L_T=\frac{1}{2} Z_T^2,$ where $T$ is the frame index and $Z_T$ is the value of the virtual queue at the beginning of the $T$th frame. We will show in Lemma \ref{le:drift} that the Lyapunov-drift satisfies
\begin{align}
&\mathbb{E}[L_{T+1}-L_T] 
\leq  \hbox{a negative drift}+H^4\iota+\epsilon^2- \frac{\eta}{K^\alpha}\sum_{k=TK^\alpha+1}^{(T+1)K^\alpha} \Phi_k ,
\label{outline:drift}
\end{align}
 where
\begin{align*}
\Phi_k=\mathbb{E}\left[ \left(\sum_a  \left\{Q_{k,1} {q}^{\epsilon,*}_1\right\}(x_{k,1},a)	- Q_{k,1} (x_{k,1},a_{k,1})\right)\right] +\mathbb{E}\left[\frac{Z_k}{\eta} \sum_a \left\{\left(C_{k,1} - {C}^{\epsilon,*}_1   \right){q}^{\epsilon,*}_1\right\}(x_{k,1},a) \right],
\end{align*} and we note that 
$\eqref{eq:(i)expanded-new}=\sum_k \Phi_k.$ Inequality \eqref{outline:drift} will be established by showing that Triple-Q takes actions to {\em almost} greedily reduce virtual-Queue $Z$ when $Z$ is large, which results in the negative drift in \eqref{outline:drift}. From \eqref{outline:drift}, we observe that 
\begin{align}
&\mathbb{E}[L_{T+1}-L_T]\leq H^4\iota+\epsilon^2-\frac{\eta}{K^\alpha}\sum_{k=TK^\alpha+1}^{(T+1)K^\alpha} \Phi_k.
\end{align}

So we can bound \eqref{eq:(i)expanded-new} by applying the telescoping sum over the $K^{1-\alpha}$ frames on the inequality above:  $$\eqref{eq:(i)expanded-new}=\sum_k\Phi_k \leq \frac{K^\alpha\mathbb{E}\left[L_1-L_{K^{1-\alpha}+1}\right]}{\eta}+\frac{K(H^4\iota+\epsilon^2)}{\eta}\leq \frac{K(H^4\iota+\epsilon^2)}{\eta},$$ where the last inequality holds because $L_1=0$ and $L_T\geq 0$ for all $T.$ Combining the bounds on \eqref{F-new} and \eqref{eq:(i)expanded-new}, we conclude that under our choices of $\iota,$ $\epsilon$ and $\eta,$
$$\eqref{step(i)}=\tilde{\cal O}(H^4S^\frac{1}{2}A^{\frac{1}{2}}K^{\frac{4}{5}}).$$
\end{itemize}
Combining the results in the three steps above, we obtain the regret bound in Theorem \ref{thm:main}. 

\subsubsection{Detailed Proof}
We next present the detailed proof. The first lemma bounds the difference between the original CMDP and its $\epsilon$-tightened version. The result is intuitive because the $\epsilon$-tightened version is a perturbation of the original problem and $\epsilon\leq \delta.$
\begin{lemma}\label{le:epsilon-dif}
Given $\epsilon\leq \delta$,  we have
	$$\mathbb{E}\left[ \sum_a \left\{Q_{1}^{*}q_1^*-  {Q}_{1}^{\epsilon,*}{q}_1^{\epsilon,*}\right\}(x_{k,1},a)\right] \leq \frac{H \epsilon }{\delta}.$$ \hfill{$\square$}
\end{lemma}
\begin{proof}
	Given ${q}_h^*(x,a)$ is the optimal solution, we have $$\sum_{h,x,a} {q}_h^*(x,a)g_h(x,a) \geq \rho.$$ Under Assumption \ref{as:1}, we know that there exists a feasible solution $\{q^{\xi_1}_h(x,a)\}_{h=1}^H$ such that
	$$\sum_{h,x,a} q_h^{\xi_1}(x,a)g_h(x,a) \geq \rho+\delta.$$ We construct $q_h^{\xi_2}(x,a) = (1-\frac{\epsilon }{\delta})q^*_h(x,a) + \frac{\epsilon}{\delta} q^{\xi_1}_h(x,a),$ which satisfies that 
	\begin{align*}
		\sum_{h,x,a} q_h^{\xi_2}(x,a)g_h(x,a) &  = \sum_{h,x,a} \left( (1-\frac{\epsilon }{\delta}) q^*_h(x,a) +\frac{\epsilon }{\delta} q^{\xi_1}_h(x,a)  \right)g_h(x,a)\geq \rho + \epsilon ,\\
		\sum_{h,x,a} q_h^{\xi_2}(x,a)  &= \sum_{x^\prime,a^\prime} p_{h-1} (x\vert x^\prime,a^\prime) q_{h-1}^{\xi_2}(x^\prime,a^\prime),\\
		\sum_{h,x,a} q_h^{\xi_2}(x,a)  &= 1.
	\end{align*}
Also we have $q_h^{\xi_2}(x,a)\geq 0 $ for all $(h,x,a).$ Thus $\{q^{\xi_2}_h(x,a)\}_{h=1}^H$ is a feasible solution to the $\epsilon$-tightened optimization problem \eqref{eq:lp-epsilon}. Then given $\{{q}^{\epsilon,*}_h(x,a)\}_{h=1}^H$ is the optimal solution to the $\epsilon$-tightened optimization problem, we have
	\begin{align*}
		& \sum_{h,x,a} \left( {q}_h^{*}(x,a) -  {q}_h^{\epsilon,*}(x,a)\right)r_h(x,a)  \\
			\leq & \sum_{h,x,a} \left( q_h^{*}(x,a) -  q_h^{\xi_2}(x,a)\right)r_h(x,a) \\
			\leq &\sum_{h,x,a} \left( q_h^{*}(x,a) -  \left(1-\frac{\epsilon }{\delta} \right)q_h^*(x,a) - \frac{\epsilon }{\delta}  q_h^{\xi_1}(x,a)\right)r_h(x,a) \\
			\leq & \sum_{h,x,a} \left( q_h^{*}(x,a) -  \left(1-\frac{\epsilon }{\delta} \right)q_h^*(x,a) \right)r_h(x,a) \\
			\leq &\frac{\epsilon }{\delta}  \sum_{h,x,a}  q_h^{*}(x,a) r_h(x,a) \\
			\leq & \frac{H\epsilon }{\delta},
	\end{align*}
where the last inequality holds because $0\leq r_h(x,a)\leq 1$ under our assumption. Therefore the result follows because 
\begin{align*}
\sum_a Q_{1}^{*}(x_{k,1},a)q_1^*(x_{k,1},a)  =&\sum_{h,x,a} {q}_h^{*}(x,a)r_h(x,a) \\
\sum_a {Q}_{1}^{\epsilon,*}(x_{k,1},a) {q}_1^{\epsilon,*}(x_{k,1},a)  =&\sum_{h,x,a} {q}_h^{\epsilon,*}(x,a)r_h(x,a).
\end{align*}
\end{proof}

The next lemma bounds the difference between the estimated Q-functions and actual Q-functions in a frame. The bound on \eqref{step:biase} is an immediate result of this lemma. 
   
\begin{lemma}\label{le:qk-qpi-bound} 

Under Triple-Q, we have for any $T\in[K^{1-\alpha}],$
	\begin{align*}
&\mathbb{E}\left[ \sum_{k=(T-1)K^\alpha+1}^{TK^\alpha}  \left\{{Q}_{k,1} - Q_{1}^{\pi_k}\right\}(x_{k,1},a_{k,1}) \right] 
\leq H^2SA  +\frac{H^3\sqrt{\iota} K^\alpha}{\chi} + \sqrt{H^2SA\iota K^{\alpha}(\chi+1)},\\
&\mathbb{E}\left[  \sum_{k=(T-1)K^\alpha+1}^{TK^\alpha}  \left\{{C}_{k,1} - C_{1}^{\pi_k}\right\}(x_{k,1},a_{k,1}) \right]
\leq  H^2SA  +\frac{H^3\sqrt{\iota} K^\alpha}{\chi} +  \sqrt{H^2SA\iota K^{\alpha}(\chi+1)}.
	\end{align*}
\end{lemma}

\begin{proof}
We will prove the result on the reward Q-function. The proof for the utility Q-function is almost identical. 
We first establish a recursive equation between a Q-function with the value-functions in the earlier episodes in the same frame. Recall that under Triple-Q,  ${Q}_{k+1,h}(x,a)$, where $k$ is an episode in frame $T,$ is updated as follows: 
\begin{align*}
	{Q}_{k+1,h}(x,a) & = \begin{cases}
		(1-\alpha_t) {Q}_{k,h}(x,a)+\alpha_t\left(r_h(x,a)+ {V}_{k,h+1} (x_{k,h+1})+b_t \right)   & \text{if $(x,a)= (x_{k,h},a_{k,h})$}\\
		{Q}_{k,h}(x,a) & \text{otherwise}
	\end{cases},
\end{align*} where $t=N_{k,h}(x,a).$ Define $k_t$ to be the index of the episode in which the agent visits $(x,a)$ in step $h$ for the $t$th time in the current frame. 

The update equation above can be written as:
\begin{align*}
	{Q}_{k,h}(x, a)  = &(1-\alpha_t) {Q}_{k_t,h}(x, a) + \alpha_t \left( r_h(x, a) + {V}_{k_t,h+1} (x_{k_t,h+1})+b_t\right).
\end{align*} 
Repeatedly using the equation above, we obtain
\begin{align}
	{Q}_{k,h}(x, a)  
	= &(1-\alpha_t)(1-\alpha_{t-1}){Q}_{k_{t-1},h}(x, a)  + (1-\alpha_t)\alpha_{t-1}\left(r_h(x, a)  + {V}_{k_{t-1},h+1}(x_{k_{t-1},h+1})+b_{t-1}\right)\nonumber\\
	&+\alpha_t \left( r_h(x, a)  + {V}_{k_t,h+1} (x_{k_t,h+1})+b_t\right)\nonumber \\
		=&\cdots\nonumber\\
	=& \alpha_t^0 Q_{(T-1)K^\alpha+1, h} (x,a)+ \sum_{i=1}^t\alpha_t^i \left(r_h(x, a) +{V}_{k_i,h+1} (x_{k_i,h+1}) + b_i  \right)\nonumber\\
		\leq & \alpha_t^0H + \sum_{i=1}^t\alpha_t^i \left(r_h(x,a)+{V}_{k_i,h+1}(x_{k_i,h+1})+ b_i\right),\end{align} where 
	$\alpha_t^0=\prod_{j=1}^t(1-\alpha_j)$ and  $\alpha_t^i=\alpha_i\prod_{j=i+1}^t(1-\alpha_j).$ From the inequality above, we further obtain 	
	\begin{align}
	 \sum_{k=(T-1)K^\alpha+1}^{TK^\alpha}  {Q}_{k,h} (x,a) \leq  \sum_{k=(T-1)K^\alpha+1}^{TK^\alpha}  \alpha_t^0 H+  \sum_{k=(T-1)K^\alpha+1}^{TK^\alpha}   \sum_{i=1}^{N_{k,h}(x,a)}\alpha_{N_{k,h}}^i \left(r_h(x,a)+{V}_{k_i,h+1}(x_{k_i,h+1})+ b_i \right).\label{eq:Q-V}
	\end{align} 

The notation becomes rather cumbersome because for each $(x_{k,h}, a_{k,h}),$ we need to consider a corresponding sequence of episode indices in which the agent sees $(x_{k,h}, a_{k,h}).$ Next we will analyze a given sample path (i.e. a specific realization of the episodes in a frame), so we simplify our notation in this proof and use the following notation: 
\begin{align*}
N_{k,h}=&N_{k,h}(x_{k,h},a_{k,h})\\
k^{(k,h)}_i=&k_i(x_{k,h},a_{k,h}),
\end{align*} where $k^{(k,h)}_i$ is the index of the episode in which the agent visits state-action pair $(x_{k,h}, a_{k,h})$ for the $i$th time. Since in a given sample path, $(k,h)$ can uniquely determine  $(x_{k,h},a_{k,h}),$  this notation introduces no ambiguity. Furthermore, we will replace $\sum_{k=(T-1)K^\alpha+1}^{TK^\alpha}$ with $\sum_k$ because we only consider episodes in frame $T$ in this proof. 

We note that
\begin{equation}
\sum_k\sum_{i=1}^{N_{k,h}} \alpha_{N_{k,h}}^i  V_{k_i^{(k,h)}, h+1}\left(x_{k_i^{(k,h)},h+1}\right)\leq \sum_k V_{k, h+1} (x_{k,h+1})\sum_{t=N_{k,h}}^\infty \alpha_t^{N_{k,h}} \leq \left(1+\frac{1}{\chi}\right)\sum_k V_{k, h+1} (x_{k,h+1}),\label{eq:Vbound}
\end{equation}
where the first inequality holds because 
because $V_{k,h+1}(x_{k,h+1})$ appears in the summation on the left-hand side each time when in episode $k'>k$ in the same frame, the environment visits $(x_{k,h},a_{k,h})$ again, i.e. $(x_{k',h},a_{k',h})=(x_{k,h},a_{k,h}),$ and the second inequality holds due to the property of the  learning rate proved in Lemma \ref{le:lr}-\ref{le:lr-d}. By substituting \eqref{eq:Vbound} into \eqref{eq:Q-V} and noting that $\sum_{i=1}^{N_{k,h}(x,a)}\alpha_{N_{k,h}}^i=1$ according to Lemma  Lemma \ref{le:lr}-\ref{le:lr-b}, we obtain
		\begin{align*}
		&\sum_k Q_{k,h} (x_{k,h},a_{k,h}) \\
		\leq &\sum_k \alpha_t^0H +\sum_k \left(r_h(x_{k,h},a_{k,h}) +V_{k,h+1}(x_{k,h+1})\right)+ \frac{1}{\chi}	\sum_k V_{k,h+1}(x_{k,h+1})+\sum_k \sum_{i=1}^{N_{k,h}}\alpha_{N_{k,h}}^i  b_i\\
		\leq &\sum_k \left(r_h(x_{k,h},a_{k,h})+V_{k,h+1}(x_{k,h+1})\right) +HSA 
		+\frac{H^2\sqrt{\iota} K^\alpha}{\chi} +  \frac{1}{2}\sqrt{H^2SA\iota K^{\alpha}(\chi+1)},
	\end{align*}
where the last inequality holds because (i) we have
\begin{align*}
\sum_k  \alpha_{N_{k,h}}^0 H =\sum_k H\mathbb{I}_{\{N_{k,h}=0\}}\leq HSA,
\end{align*} (ii) $V_{k,h+1}(x_{k,h+1}) \leq H^2\sqrt{\iota}$ by using Lemma \ref{le:q1-bound}, and (iii) we know that
\begin{align*}
&\sum_k \sum_{i=1}^{N_{k,h}} \alpha_{N_{k,h}}^i  b_i = \frac{1}{4}\sum_{k=(T-1)K^\alpha+1}^{TK^\alpha}  \sum_{i=1}^{N_{k,h}} \alpha_{N_{k,h}}^i  \sqrt{\frac{H^2\iota(\chi+1)}{\chi+i}} \leq \frac{1}{2}\sum_{k=(T-1)K^\alpha+1}^{TK^\alpha}  \sqrt{\frac{H^2\iota(\chi+1)}{\chi+N_{k,h}}}\\
= &\frac{1}{2}\sum_{x,a}\sum_{n=1}^{N_{TK^\alpha,h}(x,a)}\sqrt{\frac{H^2\iota(\chi+1)}{\chi+n}} \leq \frac{1}{2}\sum_{x,a}\sum_{n=1}^{N_{TK^\alpha,h}(x,a)}\sqrt{\frac{H^2\iota(\chi+1)}{n}} \overset{(1)}{\leq} \sqrt{H^2SA\iota K^{\alpha}(\chi+1)}, 
\end{align*}
where the last inequality  above holds because the left hand side of $(1)$ is the summation of $K^\alpha$ terms and $\sqrt{\frac{H^2\iota(\chi+1)}{\chi+n}}$ is a decreasing function of $n.$ 

Therefore, it is maximized when $N_{TK^\alpha,h} = K^\alpha / SA$ for all $x,a,$ i.e. by picking the largest $K^\alpha$ terms. Thus we can obtain
\begin{align*}
		&\sum_k Q_{k,h} (x_{k,h},a_{k,h}) -\sum_k Q_{h}^{\pi_k}(x_{k,h},a_{k,h})\\
		\leq &\sum_k\left(V_{k,h+1}(x_{k,h+1})-\mathbb P_h V_{h+1}^{\pi_k}(x_{k,h},a_{k,h})\right) +HSA  +\frac{H^2\sqrt{\iota}K^\alpha}{\chi} +  \sqrt{H^2SA\iota K^{\alpha}(\chi+1)}\\
		\leq &\sum_k\left(V_{k,h+1}(x_{k,h+1})-\mathbb P_h V_{h+1}^{\pi_k}(x_{k,h},a_{k,h})+V_{h+1}^{\pi_k}(x_{k,h+1})-V_{h+1}^{\pi_k}(x_{k,h+1})\right)\\
		&+HSA  +\frac{H^2\sqrt{\iota}K^\alpha}{\chi} +  \sqrt{H^2SA\iota K^{\alpha}(\chi+1)}\\
		=&\sum_k\left(V_{k,h+1}(x_{k,h+1}))-V_{h+1}^{\pi_k}(x_{k,h+1})-\mathbb P_h V_{h+1}^{\pi_k}(x_{k,h},a_{k,h})+\hat{\mathbb P}_h^k V^{\pi_k}_{h+1}(x_{k,h}, a_{k,h})\right)\\
		&+HSA  +\frac{H^2\sqrt{\iota}K^\alpha}{\chi} +  \sqrt{H^2SA\iota K^{\alpha}(\chi+1)}\\
		= & \sum_k\left(Q_{k,h+1}(x_{k,h+1},a_{k,h+1})-Q_{h+1}^{\pi_k}(x_{k,h+1},a_{k,h+1})-\mathbb P_h V_{h+1}^{\pi_k}(x_{k,h},a_{k,h})+\hat{\mathbb P}_h^k V^{\pi_k}_{h+1}(x_{k,h}, a_{k,h}\right)\\
		&+HSA  +\frac{H^2\sqrt{\iota}K^\alpha}{\chi} +  \sqrt{H^2SA\iota K^{\alpha}(\chi+1)}.
	\end{align*} Taking the expectation on both sides yields
\begin{align*}
		&\mathbb E\left[\sum_k Q_{k,h} (x_{k,h},a_{k,h}) -\sum_k Q_{h}^{\pi_k}(x_{k,h},a_{k,h})\right]\\
		\leq & \mathbb E\left[\sum_k\left(Q_{k,h+1}(x_{k,h+1}, a_{k,h+1}))-Q_{h+1}^{\pi_k}(x_{k,h+1}, a_{k,h+1})\right)\right]+HSA  +\frac{H^2\sqrt{\iota}K^\alpha}{\chi} + \sqrt{H^2SA\iota K^{\alpha}(\chi+1)}.
	\end{align*} 
Then by using the inequality repeatably, we obtain for any $h\in[H],$
	\begin{align*}
\mathbb E\left[\sum_k Q_{k,h} (x_{k,h},a_{k,h}) -\sum_k Q_{h}^{\pi_k}(x_{k,h},a_{k,h})\right]	\leq & H^2SA  +\frac{H^3\sqrt{\iota}K^\alpha}{\chi} +  \sqrt{H^4SA\iota K^{\alpha}(\chi+1)},
	\end{align*} so the lemma holds.

\end{proof}

From the lemma above, we can immediately conclude:
	\begin{align*}
&\mathbb{E}\left[ \sum_{k=1}^{K} \left\{Q_{k,1} - Q_{1}^{\pi_k}\right\}(x_{k,1},a_{k,1}) \right] 
\leq H^2SA K^{1-\alpha}  +\frac{H^3\sqrt{\iota}K}{\chi} +  \sqrt{H^4SA\iota K^{2-\alpha}(\chi+1)}\\
&\mathbb{E}\left[  \sum_{k=1}^{K}   \left\{C_{k,1} - C_{1}^{\pi_k}\right\}(x_{k,1},a_{k,1}) \right]
\leq   H^2SA K^{1-\alpha}  +\frac{H^3\sqrt{\iota}K}{\chi} +  \sqrt{H^4SA\iota K^{2-\alpha}(\chi+1)}.
	\end{align*}

We now focus on \eqref{step(i)}, and  further expand it as follows: 
\begin{align}
    &\eqref{step(i)}\nonumber\\
    =&\mathbb E \left[\sum_{k=1}^{K}  \left(  \sum_a \left\{{Q}_{1}^{\epsilon,*}{q}^{\epsilon,*}_1\right\}(x_{k,1},a)-Q_{k,1}(x_{k,1}, a_{k,1}) \right)\right]\nonumber\\
    =&
    \mathbb{E}\left[ \sum_{k} \sum_a \left\{\left({F}^{\epsilon,*}_{k,1} -F_{k,1}\right) {q}^{\epsilon,*}_1\right\}(x_{k,1},a) \right] \label{F}\\
	&+ \mathbb{E}\left[\sum_{k} \left(\sum_a  \left\{Q_{k,1} {q}^{\epsilon,*}_1\right\}(x_{k,1},a)	- Q_{k,1} (x_{k,1},a_{k,1})\right)\right] +\mathbb{E}\left[ \sum_{k} \frac{Z_k}{\eta} \sum_a \left\{\left(C_{k,1} - {C}^{\epsilon,*}_1   \right){q}^{\epsilon,*}_1\right\}(x_{k,1},a) \right],\label{eq:(i)expanded}
\end{align} where
\begin{align*}
	F_{k,h}(x,a) &= Q_{k,h}(x,a) +\frac{Z_k}{\eta}C_{k,h}(x,a)\\ 
	{F}_{h}^{\epsilon,*}(x,a) &=  {Q}_{h}^{\epsilon,*}(x,a)+\frac{Z_k}{\eta} C_{h}^{\epsilon,*}(x,a).
\end{align*}
We first show \eqref{F} can be bounded using the following lemma. This result holds because the choices of the UCB bonuses and the additional bonuses added at the beginning of each frame guarantee that $F_{k,h}(x,a)$ is an over-estimate of ${F}_{h}^{\epsilon,*}(x,a)$ for all $k,$ $h$ and $(x,a)$ with a high probability.  

\begin{lemma} \label{le:qk-qpi-relation} 
With probability at least $1-\frac{1}{K^3},$  the following inequality holds simultaneously for all $(x,a,h,k)\in\mathcal{S}\times\mathcal{A}\times[H]\times[K]:$
\begin{equation}
   \left\{ F_{k,h}-F_{h}^{\pi}\right\}(x,a)\geq 0,\label{eq:qk-qpi-relation-c}
\end{equation} which further implies that 
\begin{equation}
  \mathbb{E}\left[ \sum_{k=1}^{K} \sum_a \left\{\left({F}^{\epsilon,*}_{k,1} -F_{k,1}\right) {q}^{\epsilon,*}_1\right\}(x_{k,1},a) \right] \leq \frac{4H^4\iota}{\eta K}.
 \end{equation}
\label{le:qk-qpi-relation-c}
\end{lemma}
\begin{proof}
Consider frame $T$ and episodes in frame $T.$ Define $Z=Z_{(T-1)K^\alpha+1}$ because the value of the virtual queue does not change during each frame.  We further define/recall the following notations:
	\begin{align*}
		F_{k,h}(x,a) =   Q_{k,h}(x,a)+\frac{Z}{\eta}C_{k,h}(x,a), \quad  & U_{k,h}(x)=  V_{k,h}(x)+ \frac{Z}{\eta} W_{k,h}(x),\nonumber\\
		F_{h}^\pi(x,a)=   Q_{h}^\pi(x,a)+ \frac{Z}{\eta} C_{h}^\pi(x,a),\quad & U_{h}^\pi(x) =  V_{h}^\pi(x)+ \frac{Z}{\eta}W_{h}^\pi(x).
	\end{align*}
According to Lemma \ref{le:qk-qpi} in the appendix, we have
\begin{align}
&\{F_{k,h} - F_{h}^\pi\}(x,a) \nonumber\\
= &\alpha_t^0\left\{F_{(T-1)K^\alpha +1, h}-F_{h}^\pi\right\}(x,a)
	\nonumber\\
	&+\sum_{i=1}^t\alpha_t^i\left(\left\{U_{k_i,h+1}-U_{h+1}^\pi  \right\}(x_{k_i,h+1}) + \{(\hat{\mathbb{P}}_h^{k_i}-\mathbb{P}_h) U_{h+1}^{\pi}\}(x,a) + \left(1+\frac{Z}{\eta}\right)b_i\right)\nonumber\\
\geq&_{(a)}  \alpha_t^0\left\{F_{(T-1)K^\alpha+1, h}-F_{h}^\pi\right\}(x,a)
	+\sum_{i=1}^t\alpha_t^i\left\{U_{k_i,h+1}-U_{h+1}^\pi  \right\}(x_{k_i,h+1})\nonumber\\
= &_{(b)} \alpha_t^0\left\{F_{(T-1)K^\alpha+1, h}-F_{h}^\pi\right\}(x,a)
	+\sum_{i=1}^t\alpha_t^i\left(\max_a F_{k_i,h+1} (x_{k_i,h+1},a) -F_{h+1}^\pi(x_{k_i,h+1},\pi(x_{k_i,h+1})) \right)\nonumber\\
\geq &\alpha_t^0\left\{F_{(T-1)K^\alpha+1, h}-F_{h}^\pi\right\}(x,a)+ \sum_{i=1}^t\alpha_t^i\left\{F_{k_i,h+1} -F_{h+1}^\pi\right\}(x_{k_i,h+1},\pi(x_{k_i,h+1})), \label{eq:induction-F}
\end{align}
where inequality $(a)$ holds because of the concentration result in Lemma \ref{le:u-hoeffding} in the appendix and $$\sum_{i=1}^t\alpha_t^i(1 +\frac{Z}{\eta})b_i  = \frac{1}{4}\sum_{i=1}^t\alpha_t^i(1 +\frac{Z}{\eta})\sqrt{\frac{H^2\iota(\chi+1)}{\chi + i}} \geq\frac{\eta +Z}{4\eta}\sqrt{\frac{H^2\iota(\chi+1)}{\chi + t}}$$ by using Lemma \ref{le:lr}-\ref{le:lr-c}, and equality $(b)$ holds because Triple-Q selects the action that maximizes $F_{k_i,h+1} (x_{k_i,h+1},a)$ so $U_{k_i,h+1}(x_{k_i,h+1})=\max_a F_{k_i,h+1} (x_{k_i,h+1},a)$.

The inequality above suggests that we can prove $\{F_{k,h} - F_{h}^\pi\}(x,a)$ for any $(x,a)$ if (i) $$\left\{F_{(T-1)K^\alpha+1, h}-F_{h}^\pi\right\}(x,a)\geq 0,$$ i.e. the result holds at the beginning of the frame and (ii) $$\left\{F_{k',h+1} -F_{h+1}^\pi\right\}(x,a)\geq 0\quad\hbox{ for any }\quad k'<k$$ and $(x,a),$ i.e. the result holds for step $h+1$ in all the previous episodes in the same frame. 

We now prove the lemma using induction. We first consider $T=1$ and $h=H$ i.e. the last step in the first frame. In this case, inequality \eqref{eq:induction-F} becomes
\begin{align}
\{F_{k,H} - F_{H}^\pi\}(x,a) 
\geq \alpha_t^0\left\{H+\frac{Z_1}{\eta}H-F_{h}^\pi\right\}(x,a)\geq 0.
\end{align} Based on induction, we can first conclude that 
\begin{align*}
\{F_{k,h} - F_{h}^\pi\}(x,a) \geq 0
\end{align*} for all $h$ and $k\leq K^{\alpha}+1,$ where $\{F_{K^\alpha+1, h}\}_{h=1,\cdots, H}$ are the values before line 20, i.e. before adding the extra bonuses and thresholding Q-values at the end of a frame. Now suppose that \eqref{eq:qk-qpi-relation-c} holds for any episode $k$ in frame $T,$ any step $h,$ and any $(x,a).$ Now consider 
\begin{align}
\left\{F_{TK^\alpha+1, h}-F_{h}^\pi\right\}(x,a)=	 Q_{TK^\alpha+1,h}(x,a)+\frac{Z_{TK^\alpha+1}}{\eta}C_{TK^\alpha+1,h}(x,a)-Q_{h}^\pi(x,a)- \frac{Z_{TK^\alpha+1}}{\eta} C_{h}^\pi(x,a).\label{eq:new-induction}
\end{align} Note that if $Q^+_{TK^\alpha+1,h}(x,a)=C^+_{TK^\alpha+1,h}(x,a)=H,$ then $\eqref{eq:new-induction}\geq 0.$ Otherwise, from line 21-23, we have $Q^+_{TK^\alpha+1,h}(x,a)=Q^-_{TK^\alpha+1,h}(x,a)+\frac{2H^3\sqrt{\iota}}{\eta}<H$ and  $C^+_{TK^\alpha+1,h}(x,a)=C^-_{TK^\alpha+1,h}(x,a)<H.$ Here, we use superscript $-$ and $+$ to indicate the Q-values before and after 21-24 of Triple-Q. Therefore, at the beginning of frame $T+1,$ we have 
\begin{align}
\left\{F_{TK^\alpha+1, h}-F_{h}^\pi\right\}(x,a)=&	 Q^-_{TK^\alpha+1,h}(x,a)+\frac{Z}{\eta}C^-_{TK^\alpha+1,h}(x,a)-Q_{h}^\pi(x,a)- \frac{Z}{\eta} C_{h}^\pi(x,a)\nonumber\\
+& \frac{2H^3\sqrt{\iota}}{\eta}+\frac{Z_{TK^\alpha+1}-Z}{\eta}C^-_{TK^\alpha+1,h}(x,a)-\frac{Z_{TK^\alpha+1}-Z}{\eta} C_{h}^\pi(x,a)\nonumber\\
\geq_{(a)}&\frac{2H^3\sqrt{\iota}}{\eta}-2\frac{|Z_{TK^\alpha+1}-Z|}{\eta}H\nonumber\\
\geq_{(b)}& 0,\label{eq:induction-frame}
\end{align} where inequality $(a)$ holds due to the induction assumption and the fact $C^-_{TK^\alpha+1,h}(x,a)<H,$ and $(b)$ holds because according to Lemma \ref{le:q1-bound}, 
$$|Z_{TK^\alpha+1}-Z_{TK^\alpha}|\leq \max\left\{\rho+\epsilon, \frac{\sum_{k=(T-1)K^\alpha+1}^{TK^\alpha}C_{k,1}(x_{k,1},a_{k,1})}{K^\alpha}\right\}\leq H^2\sqrt{\iota}.$$

Therefore, by substituting inequality \eqref{eq:induction-frame} into inequality \eqref{eq:induction-F}, we obtain for any $TK^\alpha + 1\leq k\leq (T+1)K^\alpha+1,$ 
\begin{align}
\{F_{k,h} - F_{h}^\pi\}(x,a) \geq  \sum_{i=1}^t\alpha_t^i\left\{F_{k_i,h+1} -F_{h+1}^\pi\right\}\left(x_{k_i,h+1},\pi(x_{k_i,h+1})\right).
\end{align}
Considering $h=H,$ the inequality becomes
\begin{align}
\{F_{k,H} - F_{H}^\pi\}(x,a) \geq  0. 
\end{align} By applying induction on $h$, we conclude that \begin{align}
\{F_{k,h} - F_{h}^\pi\}(x,a) \geq  0. 
\end{align} holds for any  $TK^\alpha + 1\leq k\leq (T+1)K^\alpha+1,$ $h,$ and $(x,a),$ which completes the proof of \eqref{eq:qk-qpi-relation-c}. 

Let $\cal E$ denote the event that \eqref{eq:qk-qpi-relation-c} holds for all $k,$ $h$ and $(x,a).$ Then based on Lemma \ref{le:q1-bound}, we conclude that 
\begin{align}
  &\mathbb{E}\left[ \sum_{k=1}^{K} \sum_a \left\{\left({F}^{\epsilon,*}_{k,1} -F_{k,1}\right) {q}^{\epsilon,*}_1\right\}(x_{k,1},a) \right]\nonumber\\
  =&   \mathbb{E}\left[\left. \sum_{k=1}^{K} \sum_a \left\{\left({F}^{\epsilon,*}_{k,1} -F_{k,1}\right) {q}^{\epsilon,*}_1\right\}(x_{k,1},a) \right| {\cal E}\right]\Pr({\cal E})\nonumber\\
  &+ \mathbb{E}\left[\left. \sum_{k=1}^{K} \sum_a \left\{\left({F}^{\epsilon,*}_{k,1} -F_{k,1}\right) {q}^{\epsilon,*}_1\right\}(x_{k,1},a) \right| {\cal E}^c\right]\Pr({\cal E}^c)\nonumber\\
  \leq& 2K \left(1+\frac{K^{1-\alpha}H^2\sqrt{\iota}}{\eta}\right)H^2\sqrt{\iota}\frac{1}{K^3}\leq \frac{4H^4\iota}{\eta K}.\label{eq:F-bound}
 \end{align}
\end{proof}

Next we bound  \eqref{eq:(i)expanded} using the Lyapunov drift analysis on virtual queue $Z$. Since the virtual queue is updated every frame, we abuse the notation and define $Z_T$ to be the virtual queue used in frame $T.$ In particular, $Z_T=Z_{(T-1)K^\alpha+1}.$ We further define $$\bar{C}_T=\sum_{k=(T-1)K^\alpha+1}^{TK^\alpha}C_{k,1}(x_{k,1}, a_{k,1}).$$ Therefore, under Triple-Q, we have
\begin{align}Z_{T+1}=\left(Z_{T}+\rho+\epsilon-\frac{\bar{C}_T}{K^\alpha}\right)^+  \label{eq:zupdate}\end{align}
Define the Lyapunov function to be $$L_T = \frac{1}{2} Z_T^2.$$  The next lemma bounds the expected Lyapunov drift conditioned on  $Z_T.$

\begin{lemma}\label{le:drift}
Assume $\epsilon \leq \delta.$ The expected Lyapunov drift satisfies
\begin{align}
& \mathbb{E}\left[L_{T+1}-L_T\vert Z_T=z \right] \nonumber\\
\leq & \frac{1}{K^\alpha}\sum_{k=(T-1)K^\alpha+1}^{TK^\alpha}\left(-\eta \mathbb{E} \left[ \left.\sum_a  \left\{Q_{k,1}{q}^{\epsilon,*}_1\right\}(x_{k,1},a)-Q_{k,1} (x_{k,1},a_{k,1}) \right\vert Z_T= z \right]\right.\nonumber\\
&\left.+ z \mathbb{E}\left[\left. \sum_a  \left\{\left({C}^{\epsilon,*}_1-C_{k,1}\right){q}^{\epsilon,*}_1\right\}(x_{k,1},a)\right\vert Z_T=z  \right]\right) +H^4\iota+\epsilon^2. \label{eq:drift-inq}
\end{align}
\end{lemma}
\begin{proof}
Based on the definition of $L_T,$ the Lyapunov drift is
\begin{align*}
 L_{T+1}-L_T   \leq & Z_T\left(\rho+ \epsilon  - \frac{\bar{C}_T}{K^\alpha} \right) + \frac{\left( \frac{\bar{C}_T}{K^\alpha} +\epsilon  -\rho\right)^2}{2} \\
	\leq & Z_T\left(\rho +\epsilon  -\frac{\bar{C}_T}{K^\alpha} \right)+H^4\iota+\epsilon^2\\
	\leq & \frac{Z_T}{K^\alpha}\sum_{k=TK^\alpha+1}^{(T+1)K^\alpha} \left(\rho +\epsilon  - C_{k,1}(x_{k,1},a_{k,1}) \right)+H^4\iota+\epsilon^2
\end{align*}
where the first inequality is a result of the upper bound on $\vert C_{k,1}  (x_{k,1},a_{k,1})\vert$ in Lemma \ref{le:q1-bound}.

Let $\{q^{\epsilon}_h\}_{h=1}^H$ be a feasible solution to the tightened LP \eqref{eq:lp-epsilon}. Then the expected Lyapunov drift conditioned on $Z_T=z$ is
\begin{align}
	& \mathbb{E}\left[L_{T+1}-L_T\vert Z_T=z \right] \nonumber\\
	\leq &  \frac{1}{K^\alpha}\sum_{k=(T-1)K^\alpha+1}^{TK^\alpha}  \left(\mathbb{E}\left[\left.z\left( \rho+ \epsilon  -C_{k,1} (x_{k,1},a_{k,1})  \right) -\eta Q_{k,1}(x_{k,1},a_{k,1}) \right\vert Z_T=z \right] + \eta\mathbb{E}\left[\left. Q_{k,1} (x_{k,1},a_{k,1}) \right\vert Z_T=z  \right]\right)\nonumber\\
	&+ H^4\iota +\epsilon^2. \label{eq:drift-1}
\end{align}
Now we focus on the term inside the summation and obtain that 
\begin{align*}
	&\left(\mathbb{E}\left[\left.z\left( \rho+ \epsilon  -C_{k,1} (x_{k,1},a_{k,1})  \right) -\eta Q_{k,1}(x_{k,1},a_{k,1}) \right\vert Z_T=z \right] + \eta\mathbb{E}\left[\left. Q_{k,1} (x_{k,1},a_{k,1}) \right\vert Z_T=z  \right]\right) \\
	\leq  &_{(a)} z(\rho+\epsilon)- \mathbb{E} \left[\left.  \eta\left(\sum_a\left\{\frac{z}{\eta} C_{k,1}q^{\epsilon}_1 + Q_{k,1}q^{\epsilon}_1\right\}(x_{k,1},a)\right) \right\vert  Z_T=z  \right] + \eta\mathbb{E}\left[\left. Q_{k,1} (x_{k,1},a_{k,1}) \right\vert Z_T=z  \right]\nonumber\\
	= &  \mathbb{E} \left[ \left. z\left(  \rho  +\epsilon  -\sum_aC_{k,1} (x_{k,1},a)q^{\epsilon}_1(x_{k,1},a) \right) \right\vert Z_T=z \right]\\
	&-\mathbb{E}\left[\left. \eta \sum_a  Q_{k,1} (x_{k,1},a)q^{\epsilon}_1(x_{k,1},a)	- \eta Q_{k,1} (x_{k,1},a_{k,1})\right\vert Z_T=z\right] \nonumber\\
	=& \mathbb{E} \left[ z\left(\left.\rho +\epsilon  - \sum_a C^{\epsilon}_1 (x_{k,1},a)q^{\epsilon}_1(x_{k,1},a) \right) \right\vert Z_T=z \right]\\ &-\mathbb{E}\left[\left. \eta \sum_a  Q_{k,1} (x_{k,1},a)q^{\epsilon}_1(x_{k,1},a)- \eta Q_{k,1} (x_{k,1},a_{k,1})\right\vert Z_T=z\right] + \mathbb{E} \left[\left. z \sum_a \left\{(C^{\epsilon}_1-C_{k,1})q^{\epsilon}_1\right\}(x_{k,1},a)   \right\vert Z_T=z  \right]  \nonumber\\
	\leq & -\eta \mathbb{E} \left[\left. \sum_a  Q_{k,1} (x_{k,1},a)q^{\epsilon}_1(x_{k,1},a)	- Q_{k,1} (x_{k,1},a_{k,1}) \right\vert  Z_T=z \right]  + \mathbb{E}\left[\left. z \sum_a \left\{ (C^\epsilon_1-C_{k,1})q^{\epsilon}_1\right\}(x_{k,1},a)\right\vert  Z_T=z   \right],
\end{align*}
where inequality $(a)$ holds because $a_{k,h}$ is chosen to maximize $ Q_{k,h} (x_{k,h},a) + \frac{Z_T}{\eta} C_{k,h} (x_{k,h},a)$ under Triple-Q, and the last equality holds due to that $\{q^{\epsilon}_h(x,a)\}_{h=1}^H$ is a feasible solution to the optimization problem \eqref{eq:lp-epsilon}, so 
\begin{align*}
    \left(\rho +\epsilon  -\sum_a {C}_1^{\epsilon}(x_{k,1},a)q^{\epsilon}_1(x_{k,1},a) \right)&=\left(\rho +\epsilon  -\sum_{h,x,a}g_h(x,a){q}^{\epsilon}_h(x,a)  \right)\leq 0.
\end{align*}
Therefore, we can conclude the lemma by substituting ${q}^{\epsilon}_h(x,a)$ with the optimal solution ${q}^{\epsilon,*}_h(x,a)$.
\end{proof}

After taking expectation with respect to $Z,$ dividing $\eta$ on both sides, and then applying the telescoping sum, we obtain
\begin{align}
& \mathbb{E} \left[\sum_{k=1}^{K} \left(\sum_a  \left\{Q_{k,1} {q}^{\epsilon,*}_1\right\}(x_{k,1},a)	- Q_{k,1} (x_{k,1},a_{k,1})\right) \right] +\mathbb{E}\left[ \sum_{k=1}^{K} \frac{Z_k}{\eta} \sum_a \left\{\left(C_{k,1}  - {C}^{\epsilon,*}_1   \right){q}^{\epsilon,*}_1\right\}(x_{k,1},a) \right] \nonumber\\
\leq & \frac{K^\alpha \mathbb{E}\left[L_{1}-L_{K^{1-\alpha}+1}\right]}{\eta}  + \frac{K\left(H^4\iota+\epsilon^2\right)}{\eta}\leq   \frac{K\left(H^4\iota+\epsilon^2\right)}{\eta}\label{eq:ldrift},
\end{align}
where the last inequality holds because that $L_1=0$ and $L_{T+1}$ is non-negative.

Now combining Lemma \ref{le:qk-qpi-relation} and inequality \eqref{eq:ldrift}, we conclude that 
\begin{align*}
    \eqref{step(i)}\leq \frac{K\left(H^4\iota+\epsilon^2\right)}{\eta}+\frac{4H^4\iota}{\eta K}.
\end{align*}
Further combining inequality above with Lemma \ref{le:epsilon-dif} and Lemma \ref{le:qk-qpi-bound}, 
\begin{align}
	\text{Regret} (K)  \leq \frac{KH\epsilon}{\delta} +  H^2SA K^{1-\alpha}  +\frac{H^3\sqrt{\iota}K}{\chi} +  \sqrt{H^4SA\iota K^{2-\alpha}(\chi+1)}+ \frac{K\left(H^4\iota+\epsilon^2\right)}{\eta}+\frac{4H^4\iota}{\eta K}.
\label{eq:regret-final} 
\end{align}

By choosing $\alpha=0.6,$ i.e each frame has $K^{0.6}$ episodes, $\chi=K^{0.2},$ $\eta=K^{0.2},$ and $\epsilon = \frac{8\sqrt{SAH^6\iota^3}}{K^{0.2}},$ we conclude that  when $K \geq \left(\frac{8\sqrt{SAH^6\iota^3}}{\delta}\right)^5,$ which guarantees that $\epsilon<\delta/2,$ we have
\begin{align}
	\text{Regret} (K) \leq \frac{13}{\delta} H^4 {\sqrt{SA\iota^3}}{K^{0.8}}  +\frac{4H^4\iota}{ K^{1.2}} = \tilde{\cal O}\left(\frac{1 }{\delta}H^4 S^{\frac{1}{2}}A^{\frac{1}{2}}K^{0.8} \right).
\end{align}

\subsection{Constraint Violation}
\subsubsection{Outline of the Constraint Violation Analysis}
Again, we use $Z_T$ to denote the value of virtual-Queue in frame $T.$ According to the virtual-Queue update defined in {Triple-Q}, we have 
\begin{align*}
	Z_{T+1} = \left(  Z_T   + \rho + \epsilon -\frac{\bar{C}_T}{K^\alpha}\right)^+ 
	\geq   Z_T   + \rho + \epsilon -\frac{\bar{C}_T}{K^\alpha},
\end{align*}
which implies that 
\begin{align*}
\sum_{k=(T-1)K^\alpha+1}^{TK^\alpha}\left(-C_{1}^{\pi_k} (x_{k,1},a_{k,1}) +\rho \right) 
\leq K^\alpha\left(Z_{T+1} - Z_T \right)+ \sum_{k=(T-1)K^\alpha+1}^{TK^\alpha} \left(\left\{C_{k,1}- C_{1}^{\pi_k} \right\} (x_{k,1},a_{k,1}) - \epsilon \right).  
\end{align*}
Summing the inequality above over all frames and taking expectation on both sides, we obtain the following upper bound on the constraint violation: 
\begin{align}
	\mathbb{E} \left[\sum^{K}_{k=1}\rho-C_{1}^{\pi_k} (x_{k,1},a_{k,1})     \right]
	\leq  -K\epsilon + K^\alpha \mathbb{E}\left[ Z_{K^{1-\alpha}+1} \right]+\mathbb{E}\left[\sum_{k=1}^K\left\{C_{k,1}- C_{1}^{\pi_k}\right\} (x_{k,1},a_{k,1})\right],\label{eq:violation}
\end{align} 
where we used the fact $Z_1=0.$ 

In Lemma \ref{le:qk-qpi-bound}, we already established an upper bound on the estimation error of $C_{k,1}:$
\begin{align}
    \mathbb{E}\left[\sum_{k=1}^K\left\{C_{k,1}-C_{1}^{\pi_k}\right\} (x_{k,1},a_{k,1})\right]\leq H^2SA K^{1-\alpha}  +\frac{H^3 \sqrt{\iota} K}{\chi} +  \sqrt{H^4SA\iota K^{2-\alpha}(\chi+1)}. \label{eq:C-error}
\end{align} 
Next, we study the moment generating function of $Z_T,$ i.e. $\mathbb{E}\left[e^{rZ_T}\right]$ for some $r>0.$ Based on a Lyapunov drift analysis of this moment generating function and Jensen's inequality, we will establish the following upper bound on $Z_T$ that holds for any $1\leq T\leq K^{1-\alpha}+1$  
\begin{align}
    \mathbb{E}[ Z_{T}]  \leq &   \frac{54H^4\iota}{\delta}\log \left( \frac{16H^2\sqrt{\iota}}{\delta} \right) +\frac{16H^2\iota}{K^2\delta} + \frac{4\eta\sqrt{H^2\iota} }{\delta}.\label{eq:Z_T}
\end{align} 
Under our choices of $\epsilon,$ $\alpha,$ $\chi,$ $\eta$ and $\iota,$ it can be easily verified that $K\epsilon$ dominates the upper bounds in \eqref{eq:C-error} and \eqref{eq:Z_T}, which leads to the conclusion that the constraint violation because zero when $K$ is sufficiently large in Theorem \ref{thm:main}. 

\subsubsection{Detailed Proof}
To complete the proof, we need to establish the following upper bound on $\mathbb E[Z_{T+1}]$ based on a bound on the moment generating function. 

\begin{lemma}\label{le:zk-bound}
Assuming $\epsilon \leq \frac{\delta}{2},$ 
we have for any $1\leq T\leq K^{1-\alpha}$
\begin{align}
    \mathbb{E}[ Z_{T}]  \leq & 
   \frac{54H^4\iota}{\delta}\log \left( \frac{16H^2\sqrt{\iota}}{\delta} \right) +\frac{16H^2\iota}{K^2\delta} + \frac{4\eta\sqrt{H^2\iota} }{\delta}. \label{eq:zk-bound}
\end{align} 
\end{lemma}

The proof will also use the following lemma from \cite{Nee_16}.
\begin{lemma}\label{le:drift-bond-cond}
	Let $S_t$ be the state of a Markov chain, $L_t$ be a Lyapunov function with $L_0=l_0,$  and its drift $\Delta_t= L_{t+1}-L_t.$ Given the constant $\gamma$ and $v$ with $0<\gamma\leq v,$ suppose that the expected drift $\mathbb{E}[\Delta_t\vert S_t=s]$ satisfies the following conditions:
	\begin{itemize}
		\item[(1)] There exists constant $\gamma>0$ and $\theta_t>0$ such that $\mathbb{E}[\Delta_t\vert S_t=s]\leq -\gamma$ when $L_t\geq \theta_t.$
		\item[(2)] $\vert L_{t+1}-L_t\vert \leq v$ holds with probability one.
	\end{itemize}
	Then we have $$\mathbb{E}[e^{rL_t}]\leq e^{rl_0}+\frac{2e^{r(v+\theta_t)}}{r\gamma},$$ where $r=\frac{\gamma}{v^2+v\gamma/3}.$\hfill{$\square$}
\end{lemma}

\begin{proof}[Proof of Lemma \ref{le:zk-bound}]
We apply Lemma \ref{le:drift-bond-cond} to a new Lyapunov function: $$\bar{L}_T= Z_T.$$ 

To verify condition (1) in Lemma \ref{le:drift-bond-cond}, consider $\bar{L}_T =Z_T\geq \theta_T =\frac{4(\frac{4H^2\iota}{K^{2}}+ \eta \sqrt{H^2\iota}+ H^4\iota+\epsilon^2 )}{\delta}$ and $2\epsilon\leq \delta.$ The conditional expected drift of $\bar{L}_T$ is 
\begin{align*}
	&\mathbb{E} \left[ Z_{T+1} - Z_T \vert Z_T=z\right]\\
	=& \mathbb{E}\left[ \left. \sqrt{Z_{T+1}^2}  - \sqrt{z^2} \right\vert Z_T=z \right] \\
	\leq & \frac{1}{2z} \mathbb{E}  \left[ \left. Z_{T+1}^2  - z^2 \right\vert Z_T=z\right]\\
	\leq_{(a)}& -\frac{\delta}{2}  + \frac{\frac{4H^2\iota}{K^{2}}+ \eta \sqrt{H^2\iota}+ H^4\iota+\epsilon^2 }{z} \\
	\leq & -\frac{\delta}{2}  + \frac{\frac{4H^2\iota}{K^{2}}+ \eta \sqrt{H^2\iota}+ H^4\iota+\epsilon^2 }{\theta_T}\\
	= & -\frac{\delta}{4},
\end{align*}
where  inequality ($a$) is obtained according to Lemma \ref{le:drift_epi_neg}; and the last inequality holds given $z \geq\theta_T.$ 

To verify condition (2) in Lemma \ref{le:drift-bond-cond}, we have 
$$Z_{T+1}-Z_T \leq \vert Z_{T+1} - Z_{T}\vert \leq \left|\rho+\epsilon -\bar{C}_T\right|\leq (H^2+\sqrt{H^4\iota})+\epsilon\leq 2\sqrt{H^4\iota},$$
where the last inequality holds because $2\epsilon\leq \delta\leq 1.$

Now choose $\gamma  = \frac{\delta}{4}$ and $v=2\sqrt{H^4\iota}.$ From Lemma \ref{le:drift-bond-cond}, we obtain
\begin{align}
    \mathbb{E}\left[e^{rZ_T}\right]\leq e^{rZ_1} + \frac{2e^{r(v+\theta_T)}}{r\gamma},\quad\hbox{where}\quad r=\frac{\gamma}{v^2+v\gamma/3}.
    \label{eq:rzk-bound}
\end{align}
By Jensen's inequality, we have $$e^{r\mathbb{E}\left[ Z_T \right]  }\leq  \mathbb{E}\left[e^{r Z_T }\right],$$ which implies that
\begin{align*}
	&	\mathbb{E}[ Z_T]  \leq \frac{1}{r}\log{\left(1 + \frac{2e^{r(v+\theta_T)}}{r\gamma} \right)} \\
	= &  \frac{1}{r}\log{\left( 1 + \frac{6v^2+2v\gamma}{3\gamma^2} e^{r(v+\theta_T)}  \right)}\\
	\leq 	&  \frac{1}{r}\log{\left(1+ \frac{8v^2}{3\gamma^2} e^{r(v+\theta_T)}  \right)}\\
	\leq 	&  \frac{1}{r}\log{\left(\frac{11v^2}{3\gamma^2} e^{r(v+\theta_T)}  \right)}\\
	\leq  	&  \frac{4v^2}{3\gamma}\log{\left(\frac{11v^2}{3\gamma^2} e^{r(v+\theta_T)}  \right)}\\
	\leq & \frac{3v^2}{\gamma} \log \left(\frac{2v}{\gamma} \right) + v + \theta_T \\
	\leq&\frac{3v^2}{\gamma} \log \left(\frac{2v}{\gamma} \right)+ v + \frac{4(\frac{4H^2\iota}{K^{2}}+ \eta \sqrt{H^2\iota}+ H^4\iota+\epsilon^2 )}{\delta} \\
	= & \frac{48H^4\iota}{\delta}\log \left( \frac{16H^2\sqrt{\iota}}{\delta} \right) + 2\sqrt{H^4\iota} + \frac{4(\frac{4H^2\iota}{K^{2}}+ \eta \sqrt{H^2\iota}+ H^4\iota+\epsilon^2 )}{\delta} \\
	\leq & \frac{54H^4\iota}{\delta}\log \left( \frac{16H^2\sqrt{\iota}}{\delta} \right) +\frac{16H^2\iota}{K^2\delta} + \frac{4\eta\sqrt{H^2\iota} }{\delta}    =\tilde{\cal O}\left(\frac{\eta H}{\delta}\right),
\end{align*} which completes the proof of Lemma \ref{le:zk-bound}. 
\end{proof}

Substituting the results from Lemmas \ref{le:qk-qpi-bound}  and \ref{le:zk-bound} into \eqref{eq:violation}, under assumption $K \geq  \left(\frac{16\sqrt{SAH^6\iota^3}}{\delta} \right)^5,$ which guarantees $\epsilon\leq \frac{\delta}{2}.$ Then by using the facts that $\epsilon = \frac{8\sqrt{SAH^6\iota^3}}{ K^{0.2}},$ we can easily verify that
\begin{align*}
	&\hbox{Violation}(K)  \leq \frac{54H^4{\iota}K^{0.6} }{\delta}\log{\frac{16H^2\sqrt{\iota}}{\delta}} +\frac{4\sqrt{H^2\iota}}{\delta}K^{0.8}- 5\sqrt{SAH^6\iota^3}K^{0.8}. 
\end{align*}
If further we have $K\geq e^{\frac{1}{\delta}},$ we can obtain 
$$\hbox{Violation}(K) \leq \frac{54H^4{\iota}K^{0.6} }{\delta}\log{\frac{16H^2\sqrt{\iota}}{\delta}} - \sqrt{SAH^6\iota^3}K^{0.8} = 0.$$
Now to prove the high probability bound, recall that from inequality \eqref{eq:zupdate}, we have 
\begin{align}
	\sum^{K}_{k=1}\rho-C_{1}^{\pi_k} (x_{k,1},a_{k,1})   
	\leq  -K\epsilon + K^\alpha  Z_{K^{1-\alpha}+1} +\sum_{k=1}^K\left\{C_{k,1}- C_{1}^{\pi_k}\right\} (x_{k,1},a_{k,1}).\label{eq:violation-bound}
\end{align}  According to inequality \eqref{eq:rzk-bound}, we have
$$\mathbb{E}\left[e^{rZ_T}\right]\leq e^{rZ_1} + \frac{2e^{r(v+\theta_T)}}{r\gamma}\leq \frac{11v^2}{3\gamma^2} e^{r(v+\theta_T)},$$ which implies that 
\begin{align}
    &\Pr\left(Z_T \geq \frac{1}{r}\log \left( \frac{11v^2}{3\gamma^2}  \right) +2(v+\theta_T)\right) \nonumber\\
    = & \Pr (e^{rZ_T} \geq e^{\log \left( \frac{11v^2}{3\gamma^2}  \right) +2r(v+\theta_T)  } )\nonumber \\
    \leq & \frac{\mathbb E [e^{rZ_T}]}{ \frac{11v^2}{3\gamma^2} e^{2r(v+\theta_T)}} \nonumber \\
    \leq & \frac{1}{e^{r(v+\theta_T)}   }= \tilde{\mathcal{O}}\left(e^{-\eta} \right)\label{eq:high-1},
\end{align}
where the first inequality  is from the Markov inequality.

In the proof of Lemma \ref{le:qk-qpi-bound}, we have shown
\begin{align}
    	&\left\vert \sum_{k=(T-1)K^\alpha+1}^{TK^\alpha} C_{k,h} (x_{k,h},a_{k,h}) - C_{h}^{\pi_k}(x_{k,h},a_{k,h})\right\vert \nonumber\\
		\leq  &\left\vert \sum_{k=(T-1)K^\alpha+1}^{TK^\alpha} C_{k,h+1}(x_{k,h+1},a_{k,h+1})-C_{h+1}^{\pi_k}(x_{k,h+1},a_{k,h+1})\right\vert + \left\vert \sum_{k=(T-1)K^\alpha+1}^{TK^\alpha}  (\hat{\mathbb P}_h^k-\mathbb P_h) V_{h+1}^{\pi_k}(x_{k,h},a_{k,h})\right\vert \nonumber\\
		&+HSA  +\frac{H^2\sqrt{\iota}K^\alpha}{\chi} +  \sqrt{H^2SA\iota K^{\alpha}(\chi+1)} \label{eq:probbound}
\end{align}
Following a similar proof as the proof of Lemma \ref{le:u-hoeffding}, we can prove that $$\left\vert \sum_{k=(T-1)K^\alpha+1}^{TK^\alpha}  (\hat{\mathbb P}_h^k-\mathbb P_h) V_{h+1}^{\pi_k}(x_{k,h},a_{k,h})\right\vert \leq \frac{1}{4}\sqrt{H^2\iota K^\alpha}$$ holds with probability at least $1-\frac{1}{K^3}$.
By iteratively using inequality  \eqref{eq:probbound} over $h$ and by summing it over all frames, we conclude that with probability at at least $1-\frac{1}{K^2},$
\begin{align}
    \left\vert \sum_{k=1}^K\{ C_{k,1}-C_1^{\pi_k}\}(x_{k,1},a_{k,1})\right\vert \leq & K^{1-\alpha}H^2SA  +\frac{H^3\sqrt{\iota}K}{\chi} +  \sqrt{H^4SA\iota K^{2-\alpha}(\chi+1)}+\frac{1}{4}\sqrt{H^4\iota K^{2-\alpha}}\nonumber \\
    \leq & 4 \sqrt{H^4SA\iota}K^{0.8},\label{eq:high-2}
\end{align}
where the last inequality holds because $\alpha=0.6$ and $\chi=K^{0.2}.$ 


Now, by combining inequalities \eqref{eq:high-1} and \eqref{eq:high-2}, and using the union bound, we can show that when $K \geq \max\left\{\left(\frac{8\sqrt{SAH^6\iota^3}}{\delta}\right)^5,e^{\frac{1}{\delta}}\right\},$ with probability at least $1-\tilde{\mathcal{O}}\left(e^{-K^{0.2}} + \frac{1}{K^2}\right)$
\begin{align}
 & \sum^{K}_{k=1}\rho-C_{1}^{\pi_k} (x_{k,1},a_{k,1})  \nonumber\\
 &\leq  -K\epsilon+ K^\alpha\left(\frac{1}{r}\log \left( \frac{11v^2}{3\gamma^2}  \right) +2(v+\theta_T) \right) + 4 \sqrt{H^4SA\iota}K^{0.8} \nonumber\\
 &\leq -\sqrt{SAH^6\iota^3}K^{0.8} \leq 0,
\end{align}
which completes the proof of our main result.

 \section{Convergence and $\epsilon$-Optimal Policy}\label{sec:near}

The Triple-Q algorithm is an online learning algorithm and is not a stationary policy. In theory, we can obtain a near-optimal, stationary policy following the idea proposed in \cite{JinAllZey_18}. Assume the agent stores all the policies used during learning. Note that each policy is defined by the two Q tables and the value of the virtual queue. At the end of learning horizon $K,$ the agent constructs a stochastic policy $\bar{\pi}$ such that at the beginning of each episode, the agent uniformly and randomly selects a policy from the $K$ policies
, i.e.   $\bar{\pi}=\pi_k$ with probability $\frac{1}{K}.$ 

We note that given any initial state $x,$
$$\frac{1}{K}\sum_{k=1}^K V_{1}^{\pi_k}(x)= V^{\bar{\pi}}_1(x).$$
$$\frac{1}{K}\sum_{k=1}^K W_{1}^{\pi_k}(x)= W^{\bar{\pi}}_1(x).$$ 
Therefore, under policy $\bar{\pi},$ we have
\begin{align*}
		& \mathbb{E} \left[V_1^*(x_{k,1}) -V_1^{\bar{\pi}}(x_{k,1}) \right]\\
	= & \mathbb{E} \left[  \frac{1}{K}\sum_{k=1}^K \left(V_1^*(x_{k,1}) -V_1^{\bar{\pi}}(x_{k,1}\right)  \right] \\
	=&\mathbb{E}\left[ \frac{1}{K}\sum_{k=1}^K \left(V_1^*(x_{k,1}) -V_1^{\pi_k}(x_{k,1}\right)\right] \\
	= & \tilde{O}\left( \frac{H^4\sqrt{SA}}{\delta K^{0.2}} \right),
\end{align*}
and 
\begin{align*}
	& \mathbb{E} \left[\rho -W_1^{\bar{\pi}}(x_{k,1}) \right]\\
	= & \mathbb{E} \left[  \frac{1}{K}\sum_{k=1}^K \left(\rho -W_1^{\bar{\pi}}(x_{k,1}\right)  \right] \\
	= & \mathbb{E} \left[\frac{1}{K} \sum_{k=1}^K\left(\rho-W_1^{\pi_k}(x_{k,1}) \right)\right]\\
 \leq &0. 
\end{align*}
Therefore, given any $\epsilon,$ $\bar{\pi}$ is an $\epsilon$-optimal policy when $K$ is sufficiently large. 

While $\bar{\pi}$ is a near-optimal policy, in practice, it may not be possible to store all policies during learning due to memory constraint. A heuristic approach to obtain a near optimal, stationary policy is to fix the two Q functions (reward and utility) after learning horizon $K$ and continue to adapt the virtual queue with frame size $\sqrt{K}.$ This is a stochastic policy where the randomness comes from the virtual queue. When the virtual queue reaches its steady state, we obtain a stationary policy. The resulted policy has great performance in our experiment (see Section \ref{sec:simu}).

\section{Evaluation} \label{sec:simu}
We evaluated Triple-Q using a grid-world environment \cite{ChoNacDue_18}. We further implemented Triple-Q with neural network approximations, called Deep Triple-Q, for an environment with continuous state and action spaces, called the Dynamic Gym benchmark \cite{YanSimTin_21}. In both cases, Triple-Q and Deep Triple-Q quickly learn a safe policy with a high reward. 

\subsection{A Tabular Case} 
We first evaluated our algorithm using a grid-world environment studied in \cite{ChoNacDue_18}. The environment is shown in Figure~\ref{fig:envs}-(a). The objective of the agent is to travel to the destination as quickly as possible while avoiding obstacles for safety. Hitting an obstacle incurs a cost of $1$. The reward for the destination is $100$, and for other locations are the Euclidean distance between them and the destination subtracted from the longest distance. The cost constraint is set to be $6$ (we transferred utility to cost as we discussed in the paper), which means the agent is only allowed to hit the obstacles as most six times. To account for the statistical significance, the result of each experiment was averaged over $5$ trials. 

The result is shown in Figure~\ref{fig:re_grid}, from which we can observe that Triple-Q can quickly learn a well performed policy (with about $20,000$ episodes) while satisfying the safety constraint. Triple-Q-stop is a stationary policy obtained by stopping learning (i.e. fixing the Q tables) at $40,000$ training steps (note the virtual-Queue continues to be updated so the policy is a stochastic policy). We can see that Triple-Q-stop has similar performance as Triple-Q, and show that Triple-Q yields a near-optimal, stationary policy after the learning stops. 

\begin{figure}[htb]
	\centering
	\includegraphics[width=0.5\linewidth]{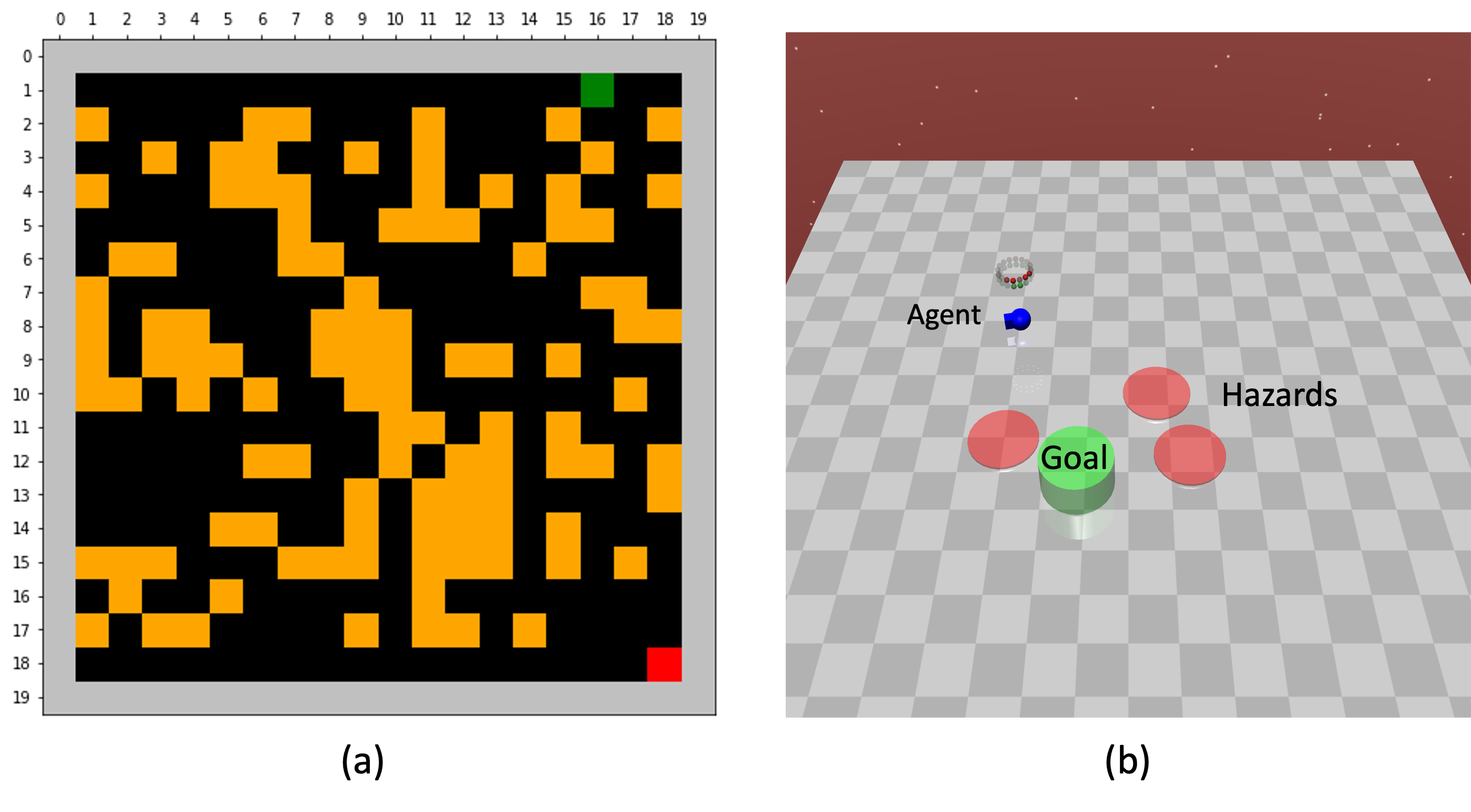}
	\caption{Grid World and DynamicEnv\protect\footnotemark \ with Safety Constraints}
	\label{fig:envs}
\end{figure}
\begin{figure}[htb]
	\centering
	\includegraphics[width=0.58\linewidth]{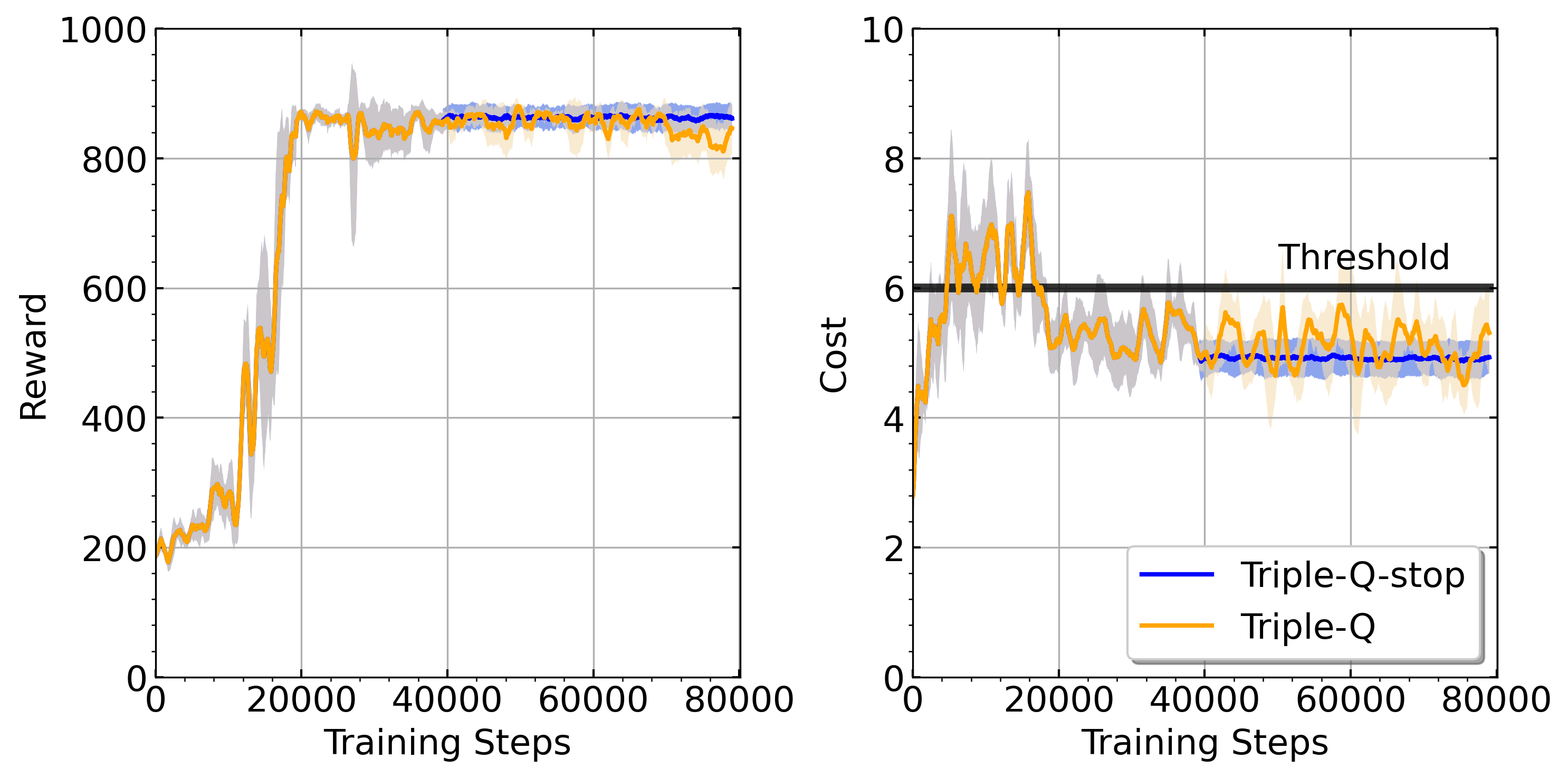}
	\caption{The average reward and cost under Triple-Q during training. The shaded region represents the standard deviations}
	\label{fig:re_grid}
\end{figure}
\footnotetext{{\bf Image Sorce:} The environment is generated using safety-gym: https://github.com/openai/safety-gym.}
\subsection{Triple-Q with Neural Networks}

We further evaluated our algorithm on the Dynamic Gym benchmark (DynamicEnv) \cite{YanSimTin_21} as shown in Figure.~\ref{fig:envs}-(b). In this environment, a point agent (one actuator for turning and another for moving) navigates on a 2D map to reach the goal position while trying to avoid reaching hazardous areas. The initial state of the agent, the goal position and hazards are randomly generated in each episode. At each step, the agents get a cost of $1$ if it stays in the hazardous area; and otherwise, there is no cost. The constraint is that the expected cost should not exceed 15. In this environment, both the states and action spaces are continuous, we implemented the key ideas of Triple-Q with neural network approximations and the actor-critic method. In particular, two Q functions are trained simultaneously, the virtual queue is updated slowly every few episodes, and the policy network is trained by optimizing the combined three ``Q''s (Triple-Q). The details can be found in Table \ref{tab:hyper}. We call this algorithm Deep Triple-Q. The simulation results in Figure~\ref{fig:re_dynamic} show that Deep Triple-Q learns a safe-policy with a high reward much faster than WCSAC \cite{YanSimTin_21}. In particular, it took around $0.45$ million training steps under Deep Triple-Q, but it took $4.5$ million training steps under WCSAC.

\begin{table}[htb]
    \centering
    \caption{Hyperparameters}
    \begin{tabular}{l|c}
       Parameter &  Value\\
       \hline
      \quad optimizer & Adam \\
      \quad  learning rate & $3\times 1^{-3}$\\
      \quad  discount & 0.99 \\
      \quad replay buffer size & $10^6$ \\
      \quad number of hidden layers (all networks) & 2 \\
      \quad batch Size & 256 \\
      \quad nonlinearity & ReLU \\
      \quad number of hidden units per layer (Critic) & 256 \\
      \quad number of hidden units per layer (Actor) & 256 \\
      \quad virtual queue update frequency & 3 episode \\
      \hline
    \end{tabular}
    \label{tab:hyper}
\end{table}

\begin{figure}[htb]
	\centering
	\includegraphics[width=0.58\linewidth]{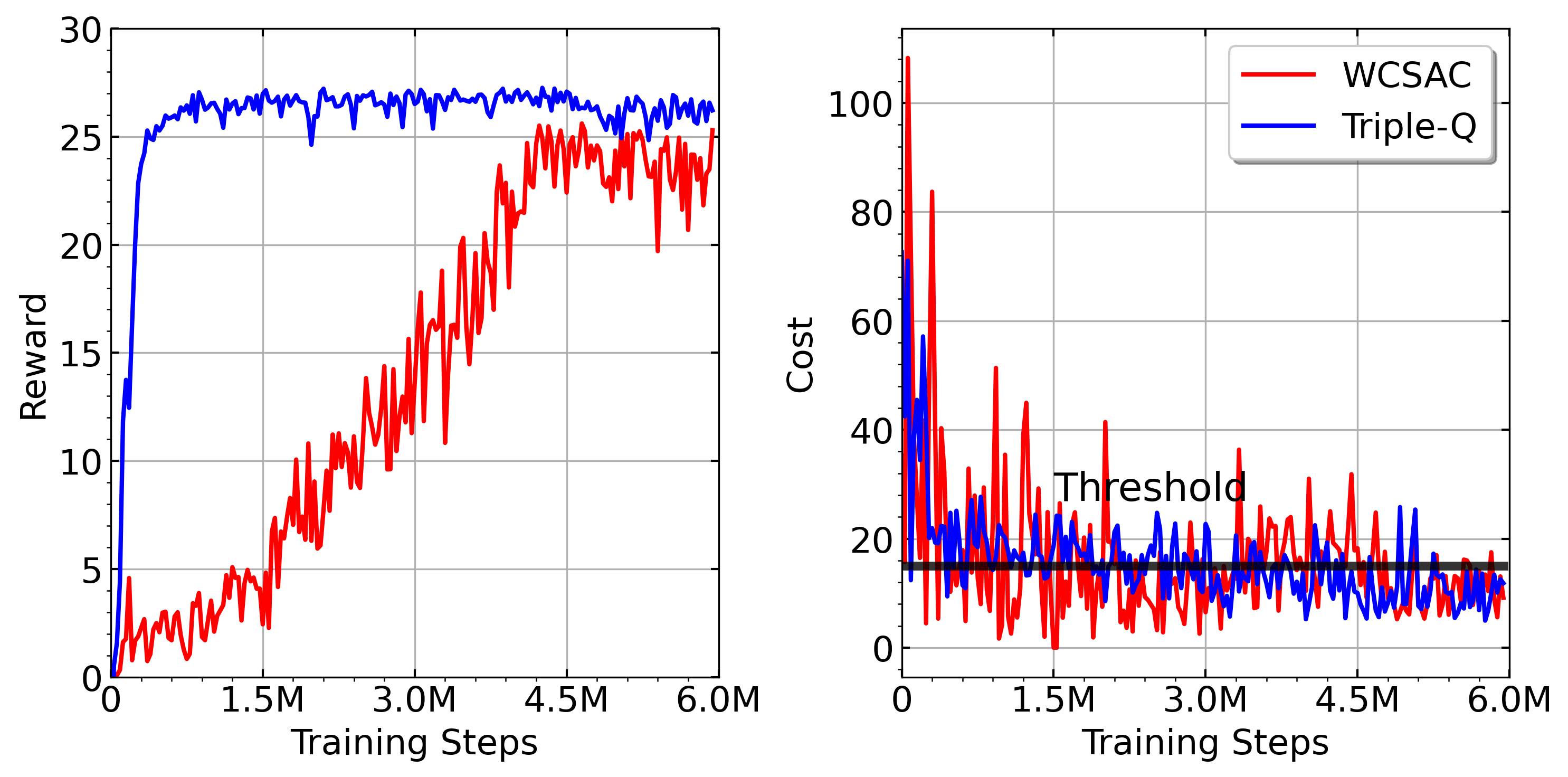}
	\caption{The rewards and costs of Deep Triple-Q versus WCSAC during Training}
	\label{fig:re_dynamic}
\end{figure}
\section{Conclusions}\label{sec:col}
This paper considered CMDPs and proposed a model-free RL algorithm without a simulator, named {Triple-Q}. From a theoretical perspective, {\em Triple-Q} achieves sublinear regret and {\em zero} constraint violation. We believe it is the first {\em model-free} RL algorithm for CMDPs with provable sublinear regret, without a simulator. From an algorithmic perspective, {Triple-Q} has similar computational complexity with SARSA, and can  easily incorporate recent deep Q-learning algorithms to obtain a deep {\em Triple-Q} algorithm, which makes our method particularly appealing for complex and challenging CMDPs in practice. 

While we only considered a single constraint in the paper, it is straightforward to extend the algorithm and the analysis to multiple constraints. Assuming there are $J$ constraints in total, Triple-Q can maintain a virtual queue and a utility Q-function for each constraint, and then selects an action at each step by solving the following problem:
$$\max_a \left(Q_{h} (x_{h},a) + \frac{1}{\eta}\sum_{j=1}^J {Z^{(j)}} C^{(j)}_{h}(x_{h},a)\right).$$ 
\bibliographystyle{plain}

\begin{thebibliography}{10}

\bibitem{AbePreCez_10}
Naoki Abe, Prem Melville, Cezar Pendus, Chandan~K Reddy, David~L Jensen,
  Vince~P Thomas, James~J Bennett, Gary~F Anderson, Brent~R Cooley, Melissa
  Kowalczyk, et~al.
\newblock Optimizing debt collections using constrained reinforcement learning.
\newblock In {\em Proceedings of the 16th ACM SIGKDD international conference
  on Knowledge discovery and data mining}, pages 75--84, 2010.

\bibitem{AchHelDav_17}
Joshua Achiam, David Held, Aviv Tamar, and Pieter Abbeel.
\newblock Constrained policy optimization.
\newblock In {\em International Conference on Machine Learning}, pages 22--31.
  PMLR, 2017.

\bibitem{Alt_99}
Eitan Altman.
\newblock {\em Constrained Markov decision processes}, volume~7.
\newblock CRC Press, 1999.

\bibitem{AzaRemHil_12}
Mohammad~Gheshlaghi Azar, R\'{e}mi Munos, and Hilbert~J. Kappen.
\newblock On the sample complexity of reinforcement learning with a generative
  model.
\newblock In {\em Int. Conf. Machine Learning (ICML)}, Madison, WI, USA, 2012.

\bibitem{AzaRemHil_13}
Mohammad~Gheshlaghi Azar, R\'{e}mi Munos, and Hilbert~J. Kappen.
\newblock Minimax pac bounds on the sample complexity of reinforcement learning
  with a generative model.
\newblock {\em Mach. Learn.}, 91(3):325–349, June 2013.

\bibitem{BraDudLyk_20}
Kiant{\'e} Brantley, Miroslav Dudik, Thodoris Lykouris, Sobhan Miryoosefi, Max
  Simchowitz, Aleksandrs Slivkins, and Wen Sun.
\newblock Constrained episodic reinforcement learning in concave-convex and
  knapsack settings.
\newblock {\em arXiv preprint arXiv:2006.05051}, 2020.

\bibitem{CheDonWan_21}
Yi~Chen, Jing Dong, and Zhaoran Wang.
\newblock A primal-dual approach to constrained {Markov} decision processes.
\newblock {\em arXiv preprint arXiv:2101.10895}, 2021.

\bibitem{ChoNacDue_18}
Yinlam Chow, Ofir Nachum, Edgar Duenez-Guzman, and Mohammad Ghavamzadeh.
\newblock A {l}yapunov-based approach to safe reinforcement learning.
\newblock {\em arXiv preprint arXiv:1805.07708}, 2018.

\bibitem{DinWeiYan_20}
Dongsheng Ding, Xiaohan Wei, Zhuoran Yang, Zhaoran Wang, and Mihailo Jovanovic.
\newblock Provably efficient safe exploration via primal-dual policy
  optimization.
\newblock In {\em Proceedings of The 24th International Conference on
  Artificial Intelligence and Statistics}, 2021.

\bibitem{DinZhaBas_20}
Dongsheng Ding, Kaiqing Zhang, Tamer Basar, and Mihailo Jovanovic.
\newblock Natural policy gradient primal-dual method for constrained markov
  decision processes.
\newblock {\em Advances in Neural Information Processing Systems}, 33, 2020.

\bibitem{EfrManPir_20}
Yonathan Efroni, Shie Mannor, and Matteo Pirotta.
\newblock Exploration-exploitation in constrained {MDP}s.
\newblock {\em arXiv preprint arXiv:2003.02189}, 2020.

\bibitem{FisAkaZei_18}
Jaime~F Fisac, Anayo~K Akametalu, Melanie~N Zeilinger, Shahab Kaynama, Jeremy
  Gillula, and Claire~J Tomlin.
\newblock A general safety framework for learning-based control in uncertain
  robotic systems.
\newblock {\em IEEE Transactions on Automatic Control}, 64(7):2737--2752, 2018.

\bibitem{GarFer_12}
Javier Garcia and Fernando Fern{\'a}ndez.
\newblock Safe exploration of state and action spaces in reinforcement
  learning.
\newblock {\em Journal of Artificial Intelligence Research}, 45:515--564, 2012.

\bibitem{JinAllZey_18}
Chi Jin, Zeyuan Allen-Zhu, Sebastien Bubeck, and Michael~I Jordan.
\newblock Is q-learning provably efficient?
\newblock In {\em Advances Neural Information Processing Systems (NeurIPS)},
  volume~31, pages 4863--4873, 2018.

\bibitem{KriRahPie_20}
Krishna~C. Kalagarla, Rahul Jain, and Pierluigi Nuzzo.
\newblock A sample-efficient algorithm for episodic finite-horizon {MDP} with
  constraints.
\newblock {\em arXiv preprint arXiv:2009.11348}, 2020.

\bibitem{Nee_16}
M.~J. {Neely}.
\newblock Energy-aware wireless scheduling with near-optimal backlog and
  convergence time tradeoffs.
\newblock {\em IEEE/ACM Transactions on Networking}, 24(4):2223--2236, 2016.

\bibitem{Nee_10}
Michael~J. Neely.
\newblock Stochastic network optimization with application to communication and
  queueing systems.
\newblock {\em Synthesis Lectures on Communication Networks}, 3(1):1--211,
  2010.

\bibitem{OnoPavKuw_15}
Masahiro Ono, Marco Pavone, Yoshiaki Kuwata, and J~Balaram.
\newblock Chance-constrained dynamic programming with application to risk-aware
  robotic space exploration.
\newblock {\em Autonomous Robots}, 39(4):555--571, 2015.

\bibitem{PatChaCal_19}
Santiago Paternain, Luiz Chamon, Miguel Calvo-Fullana, and Alejandro Ribeiro.
\newblock Constrained reinforcement learning has zero duality gap.
\newblock In {\em Advances in Neural Information Processing Systems}, 2019.

\bibitem{PutMar_14}
Martin~L Puterman.
\newblock {\em Markov decision processes: discrete stochastic dynamic
  programming}.
\newblock John Wiley \& Sons, 2014.

\bibitem{QiuWeiYan_20}
Shuang Qiu, Xiaohan Wei, Zhuoran Yang, Jieping Ye, and Zhaoran Wang.
\newblock Upper confidence primal-dual reinforcement learning for {CMDP} with
  adversarial loss.
\newblock In {\em Advances in Neural Information Processing Systems}, 2020.

\bibitem{RumNir_94}
Gavin~A Rummery and Mahesan Niranjan.
\newblock {\em On-line Q-learning using connectionist systems}, volume~37.
\newblock University of Cambridge, Department of Engineering Cambridge, UK,
  1994.

\bibitem{SinGupShr_20}
Rahul Singh, Abhishek Gupta, and Ness~B Shroff.
\newblock Learning in markov decision processes under constraints.
\newblock {\em arXiv preprint arXiv:2002.12435}, 2020.

\bibitem{SriYin_14}
R.~Srikant and Lei Ying.
\newblock {\em Communication Networks: {A}n Optimization, Control and
  Stochastic Networks Perspective}.
\newblock Cambridge University Press, 2014.

\bibitem{SutBar_18}
Richard~S Sutton and Andrew~G Barto.
\newblock {\em Reinforcement learning: An introduction}.
\newblock MIT press, 2018.

\bibitem{WanDonChe_20}
Yuanhao Wang, Kefan Dong, Xiaoyu Chen, and Liwei Wang.
\newblock Q-learning with {UCB} exploration is sample efficient for
  infinite-horizon {MDP}.
\newblock In {\em International Conference on Learning Representations}, 2020.

\bibitem{Wat_89}
Christopher John Cornish~Hellaby Watkins.
\newblock {\em Learning from Delayed Rewards}.
\newblock PhD thesis, King's College, King's College, Cambridge United Kingdom,
  May 1989.

\bibitem{WeiJahLuo_20}
Chen-Yu Wei, Mehdi~Jafarnia Jahromi, Haipeng Luo, Hiteshi Sharma, and Rahul
  Jain.
\newblock Model-free reinforcement learning in infinite-horizon average-reward
  markov decision processes.
\newblock In {\em International Conference on Machine Learning}, pages
  10170--10180. PMLR, 2020.

\bibitem{XuLiaLan_20}
Tengyu Xu, Yingbin Liang, and Guanghui Lan.
\newblock A primal approach to constrained policy optimization: Global
  optimality and finite-time analysis.
\newblock {\em arXiv preprint arXiv:2011.05869}, 2020.

\bibitem{YanSimTin_21}
Qisong Yang, Thiago~D Sim{\~a}o, Simon~H Tindemans, and Matthijs~TJ Spaan.
\newblock Wcsac: Worst-case soft actor critic for safety-constrained
  reinforcement learning.
\newblock In {\em Proceedings of the Thirty-Fifth AAAI Conference on Artificial
  Intelligence. AAAI Press, online}, 2021.

\bibitem{YuLiuNem_19}
Chao Yu, Jiming Liu, and Shamim Nemati.
\newblock Reinforcement learning in healthcare: {A} survey.
\newblock {\em arXiv preprint arXiv:1908.08796}, 2020.

\end{thebibliography}

\newpage
\appendix
In the appendix, we summarize notations used throughout the paper in Table \ref{ta:notations}, and present a few lemmas used to prove the main theorem. 
\section{Notation Table}
\begin{table}[bh]
	\caption{Notation Table}
	\label{ta:notations}
	\vskip 0.10in
	\begin{center}
		\begin{tabular}{c|l}
			\toprule
			Notation & Definition  \\
			\midrule
			$ K $    &   The total number  of episodes\\
			\hline
			$ S$    &   The number of states\\
			\hline
			$ A$    &   The number of actions\\
			\hline
			$ H$    &   The length of each episode\\
			\hline
			$ [H]$    &   Set $\{1,2,\dots,H\}$\\
			\hline
			$Q_{k,h}(x,a)$ & The estimated reward Q-function at step $h$ in episode $k$ \\
			\hline
			$Q_{h}^\pi (x,a)$ & The reward Q-function at step $h$ in episode $k$ under policy $\pi$\\
			\hline
			$V_{k,h}(x)$ & The estimated  reward value-function at step $h$ in episode $k$. \\
			\hline
			$V_{h}^\pi (x)$ & The value-function at step $h$ in episode $k$ under policy $\pi$\\
			\hline
			$C_{k,h}(x,a)$ & The estimated utility Q-function at step $h$ in episode $k$ \\
			\hline
			$C_{h}^\pi (x,a)$ & The utility Q-function at step $h$ in episode $k$ under policy $\pi$\\
			\hline
			$W_{k,h}(x)$ & The estimated utility value-function at step $h$ in episode $k$  \\
			\hline
			$W_{h}^\pi (x)$ & The utility value-function at step $h$ in episode $k$ under policy $\pi$\\
			\hline
			$F_{k,h}(x,a)$ & 
			$F_{k,h} (x,a)= Q_{k,h}(x,a) + \frac{Z_k}{\eta} C_{k,h}(x,a)$  \\
			\hline
			$U_{k,h}(x)$ &
			$U_{k,h} (x)=V_{k,h}(x) + \frac{Z_k}{\eta} W_{k,h}(x)$  \\
			\hline
			$r_h(x,a)$ & The reward of (state, action) pair $(x,a)$ at step $h.$ \\
			\hline
			$g_h(x,a)$ & The utility of (state, action) pair $(x,a)$ at step $h.$ \\
			\hline
			$N_{k,h}(x,a)$ & The number of visits to $(x,a)$ when at step $h$ in episode $k$ (not including $k$) \\ 
			\hline
			$Z_k$ & The dual estimation (virtual queue) in episode $k.$\\
			\hline
			$q_h^*$ & The optimal solution to the LP of the CMDP \eqref{eq:lp}. \\
			\hline
			${q}_h^{\epsilon,*}$ & The optimal solution to the tightened LP \eqref{eq:lp-epsilon}. \\
			\hline
			$\delta$ & Slater's constant.\\
			\hline
				$b_t$ & the UCB bonus for given $t$\\
			\hline
			$\mathbb{I}(\cdot)$ & The indicator function\\
			\hline
			\bottomrule
		\end{tabular}
	\end{center}
\end{table}

\section{Useful Lemmas}

The first lemma establishes some key properties of the learning rates used in Triple-Q.  The proof closely follows the proof of Lemma 4.1 in \cite{JinAllZey_18}. 
\begin{lemma}\label{le:lr}
Recall that the learning rate used in Triple-Q is $\alpha_t = \frac{\chi+1}{\chi+t},$ and \begin{equation}
	\alpha_t^0=\prod_{j=1}^t(1-\alpha_j)\quad \hbox{and}\quad \alpha_t^i=\alpha_i\prod_{j=i+1}^t(1-\alpha_j). \label{le:lr-def}
\end{equation} 
The following properties hold for $\alpha_t^i:$ 
	\begin{enumerate}[label=(\alph*)]
		\item $\alpha_t^0=0$ for $t \geq 1, \alpha_t^0=1$ for $t=0.$ \label{le:lr-a}
		\item $\sum_{i=1}^t\alpha_t^i=1$ for $t\geq 1,$ $\sum_{i=1}^{t}\alpha_t^i=0$ for $t=0.$\label{le:lr-b}
		\item $\frac{1}{\sqrt{\chi+t}}\leq \sum_{i=1}^t \frac{\alpha_t^i}{\sqrt{\chi + i}}\leq \frac{2}{\sqrt{\chi+t}}.$ \label{le:lr-c}
		\item $\sum_{t=i}^\infty\alpha_t^i=1+\frac{1}{\chi}$ for every $i\geq 1.$\label{le:lr-d}
		\item $ \sum_{i=1}^t (\alpha_t^i)^2\leq \frac{\chi+1}{\chi+t}$ for every $t\geq 1.$ \label{le:lr-e}
	\end{enumerate}\hfill{$\square$}
\end{lemma}
\begin{proof} The proof of \ref{le:lr-a} and \ref{le:lr-b} are straightforward by using the definition of $\alpha_t^i$. The proof of \ref{le:lr-d} is the same as that in \cite{JinAllZey_18}.

\ref{le:lr-c}: We next prove \ref{le:lr-c} by induction. 

For $t=1,$ we have $\sum_{i=1}^t\frac{\alpha_t^i}{\sqrt{\chi+i}}=\frac{\alpha_1^1}{\sqrt{\chi+1}}=\frac{1}{\sqrt{\chi+1}},$ so \ref{le:lr-c} holds for $t=1$.

Now suppose that \ref{le:lr-c} holds for $t-1$ for $t\geq 2,$ i.e. $$\frac{1}{\sqrt{\chi+t-1}}\leq \sum_{i=1}^{t-1} \frac{\alpha_t^i}{\sqrt{\chi + i-1}}\leq \frac{2}{\sqrt{\chi+t-1}}.$$
From the relationship $\alpha_t^i = (1-\alpha_t)\alpha_{t-1}^i$ for $i=1,2,\dots,t-1,$ we have $$\sum_{i=1}^t\frac{\alpha_t^i}{\chi + i} =\frac{\alpha_t}{\sqrt{\chi+t}}+(1-\alpha_t)\sum_{i=1}^{t-1}\frac{\alpha_{t-1}^i}{\sqrt{\chi+i}}.$$

Now we apply the induction assumption. To prove the lower bound in \ref{le:lr-c}, we have
$$\frac{\alpha_t}{\sqrt{\chi+t}}+(1-\alpha_t)\sum_{i=1}^{t-1}\frac{\alpha_{t-1}^i}{\sqrt{\chi+i}}\geq \frac{\alpha_t}{\sqrt{\chi+t}} + \frac{1-\alpha_t}{ \sqrt{\chi +t- 1}}\geq \frac{\alpha_t}{\sqrt{\chi+t}} + \frac{1-\alpha_t}{ \sqrt{\chi+t}}\geq \frac{1}{\sqrt{\chi+t}}.$$
To prove the upper bound in \ref{le:lr-c}, we have
\begin{align}
\frac{\alpha_t}{\sqrt{\chi+t}}+(1-\alpha_t)\sum_{i=1}^{t-1}\frac{\alpha_{t-1}^i}{\sqrt{\chi+i}} \leq & \frac{\alpha_t}{\sqrt{\chi+t}} + \frac{2(1-\alpha_t)}{\sqrt{\chi+t-1}} = \frac{\chi+1}{(\chi+t)\sqrt{\chi+t}} + \frac{2(t-1)}{(\chi+t)\sqrt{\chi+t-1}},\nonumber\\
=& \frac{1-\chi-2t}{(\chi+t)\sqrt{\chi+t}}+ \frac{2(t-1)}{(\chi+t)\sqrt{\chi+t-1}} +\frac{2}{\sqrt{\chi+t}} \nonumber\\
\leq & \frac{-\chi-1}{(\chi+t)\sqrt{\chi+t-1}}+\frac{2}{\sqrt{\chi+t}} \leq \frac{2}{\sqrt{\chi+t}}.
\end{align}

\ref{le:lr-e} According to its definition, we have 
\begin{align}
\alpha_t^i  =& \frac{\chi+1}{i+\chi}\cdot \left(\frac{i}{i+1+\chi}\frac{i+1}{i+2+\chi}\cdots \frac{t-1}{t+\chi} \right)\nonumber\\
= & \frac{\chi+1}{t+\chi}\cdot \left(\frac{i}{i+\chi}\frac{i+1}{i+1+\chi}\cdots \frac{t-1}{t-1+\chi} \right) \leq \frac{\chi+1}{\chi+t}.
\end{align}
Therefore, we have $$\sum_{i=1}^t (\alpha_t^i)^2 \leq [\max_{i\in[t]}\alpha_t^i]\cdot \sum_{i=1}^t\alpha_t^i\leq \frac{\chi+1}{\chi+t},$$ because $\sum_{i=1}^t\alpha_t^i=1.$
\end{proof}
The next lemma establishes upper bounds on $Q_{k,h}$ and $C_{k,h}$ under Triple-Q. 
\begin{lemma}\label{le:q1-bound}
	For any $(x,a,h,k)\in\mathcal{S}\times\mathcal{A}\times[H]\times[K],$ we have the following bounds on $Q_{k,h}(x,a)$ and $C_{k,h}(x,a):$
	\begin{align*}
	0\leq Q_{k,h}(x,a)\leq H^2\sqrt{\iota}\\
	0\leq C_{k,h}(x,a)\leq H^2\sqrt{\iota}.
	\end{align*}
\end{lemma}

\begin{proof}
We first consider the last step of an episode, i.e. $h=H.$ 
Recall that $V_{k, H+1}(x)=0$ for any $k$ and $x$ by its definition and $Q_{0,H}=H\leq H\sqrt{\iota}.$ Suppose $Q_{k',H}(x,a)\leq H\sqrt{\iota}$ for any $k'\leq k-1$ and any $(x,a).$ Then,
$${Q}_{k,H}(x,a)= (1-\alpha_t)Q_{k_t,H}(x,a) + \alpha_t\left(r_H(x,a)+b_t\right)\leq \max\left\{H\sqrt{\iota}, 1+\frac{H\sqrt{\iota}}{4}\right\}\leq H\sqrt{\iota},$$  where $t=N_{k,H}(x,a)$ is the number of visits to state-action pair $(x,a)$ when in step $H$ by episode $k$ (but not include episode $k$) and $k_t$ is the index of the episode of the most recent visit.  Therefore, the upper bound holds for $h=H.$

Note that $Q_{0,h}=H\leq H(H-h+1)\sqrt{\iota}.$ 
Now suppose the upper bound holds for $h+1,$ and also holds for $k'\leq k-1.$ Consider step $h$ in episode $k:$  
\begin{align*}
	{Q}_{k,h}(x,a)= &(1-\alpha_t)Q_{k_t,  h}(x,a) + \alpha_t\left(r_{h}(x,a)+V_{k_t, {h}+1}(x_{k_t, {h}+1})+b_t\right),
\end{align*} where $t=N_{k,{h}}(x,a)$ is the number of visits to state-action pair $(x,a)$ when in step ${h}$ by episode $k$ (but not include episode $k$) and $k_t$ is the index of the episode of the most recent visit.  We also note that $V_{k,h+1}(x)\leq \max_a Q_{k,h+1}(x,a)\leq H(H-h)\sqrt{\iota}.$
Therefore, we obtain
\begin{align*}
	{Q}_{k,h}(x,a)\leq \max \left\{H(H-h+1)\sqrt{\iota}, 1+H(H-h)\sqrt{\iota}+\frac{H\sqrt{\iota}}{4}\right\}\leq H(H-h+1)\sqrt{\iota}.
\end{align*} Therefore, we can conclude that $Q_{k,h}(x,a)\leq H^2\sqrt{\iota}$ for any $k,$ $h$ and $(x,a).$  The proof for $C_{k,h}(x,a)$ is identical. 
\end{proof}

Next, we present the following lemma from \cite{JinAllZey_18}, which establishes a recursive relationship between $Q_{k,h}$ and $Q^\pi_h$ for any $\pi.$ We include the proof so the paper is self-contained. 
\begin{lemma}\label{le:qk-qpi}
	Consider any $(x,a,h,k)\in\mathcal{S}\times\mathcal{A}\times[H]\times[K],$ and any policy $\pi.$ 
	Let t=$N_{k,h}(x,a)$ be the number of visits to $(x,a)$ when at step $h$ in frame $T$ before episode $k,$ and $k_1,\dots,k_t$ be the indices of the episodes in which these visits occurred. We have the following two equations:  
	\begin{align}
	(Q_{k,h} - Q_{h}^{\pi})(x,a) =& \alpha_t^0\left\{Q_{(T-1)K^\alpha+1,h}-Q_{h}^{\pi}\right\}(x,a)\nonumber \\
	&+ \sum_{i=1}^t\alpha_t^i\left( \left\{V_{k_i,h+1}-V_{h+1}^{\pi}\right\}(x_{k_i,h+1})+\left\{\hat{\mathbb{P}}_h^{k_i}V_{h+1}^{\pi}-\mathbb{P}_hV_{h+1}^{\pi} \right\}(x,a)+b_i\right), \label{le:qk-qki-1}\\
	(C_{k,h} - C_{h}^{\pi})(x,a)=& \alpha_t^0\left\{C_{(T-1)K^\alpha+1,h}-C_{h}^{\pi}\right\}(x,a)\nonumber\\
	&+ \sum_{i=1}^t\alpha_t^i \left(\left\{W_{k_i,h+1} - W_{h+1}^{\pi}\right\}(x_{k_i,h+1})+\left\{\hat{\mathbb{P}}_h^{k_i}W_{h+1}^\pi -\mathbb{P}_h W_{h+1}^{\pi}\right\}(x,a)+b_i\right),\label{le:qk-qki-2}
	\end{align}
	where $\hat{\mathbb{P}}_h^kV_{h+1}(x,a):=V_{h+1}(x_{k,h+1})$ is the empirical counterpart of $\mathbb{P}_hV_{h+1}^\pi(x,a)=\mathbb{E}_{x^\prime\sim\mathbb{P}_h(\cdot\vert x,a)}V_{h+1}^\pi (x^\prime).$ This definition can also be applied to $W_h^\pi$ as well.
\end{lemma}
\begin{proof} We will prove \eqref{le:qk-qki-1}. The proof for \eqref{le:qk-qki-2} is identical.  Recall that under Triple-Q,  $Q_{k+1,h}(x,a)$ is updated as follows: 
\begin{align*}
	Q_{k+1,h}(x,a) & = \begin{cases}
		(1-\alpha_t)Q_{k,h}(x,a)+\alpha_t\left(r_h(x,a)+V_{k,h+1} (x_{h+1,k})+b_t \right)   & \text{if $(x,a)= (x_{k,h},a_{k,h})$}\\
		Q_{k,h}(x,a) & \text{otherwise}
	\end{cases}.
\end{align*} 

From the update equation above, we have in episode $k,$ 
\begin{align*}
	Q_{k,h}(x, a)  = &(1-\alpha_t) Q_{k_t,h}(x, a) + \alpha_t \left( r_h(x, a) + V_{k_t,h+1} (x_{k_t,h+1})+b_t\right).
\end{align*} 

Repeatedly using the equation above, we obtain
\begin{align}
	Q_{k,h}(x, a)  
	= &(1-\alpha_t)(1-\alpha_{t-1})Q_{k_{t-1},h}(x, a)  + (1-\alpha_t)\alpha_{t-1}\left(r_h(x, a)  + V_{k_{t-1},h+1}(x_{k_{t-1},h+1})+b_{t-1}\right)\nonumber\\
	&+\alpha_t \left( r_h(x, a)  + V_{k_t,h+1} (x_{k_t,h+1})+b_t\right)\nonumber \\
		=&\cdots\nonumber\\
	=& \alpha_t^0 Q_{(T-1)K^\alpha+1,h}(x,a) + \sum_{i=1}^t\alpha_t^i \left(r_h(x, a) +V_{k_i,h+1} (x_{k_i,h+1}) + b_i  \right),\label{eq:qupdate}
\end{align} where
the last equality holds due to  the definition of $\alpha_t^i$  in \eqref{le:lr-def} and the fact that all $Q_{1,h}(x,a)$s are initialized to be $H.$ Now applying the Bellman equation $Q_{h}^\pi(x,a) = \left\{r_h + \mathbb{P}_hV_{h+1}^\pi\right\}(x,a)$ and the fact that $\sum_{i=1}^t\alpha_t^i=1,$ we can further obtain
\begin{align}
	Q_{h}^{\pi}(x,a)  &=  \alpha_t^0 Q_{h}^{\pi}(x,a) + (1-\alpha_t^0 )Q_{h}^{\pi}(x,a) \nonumber \\
	&= \alpha_t^0 Q_{h}^{\pi}(x,a) + \sum_{i=1}^t\alpha_t^i\left(r(x,a) + \mathbb{P}_hV_{h+1}^{\pi}(x,a)  + V_{h+1}^{\pi}(x_{k_i,h+1}) - V_{h+1}^{\pi}(x_{k_i,h+1}) \right)\nonumber\\
	&= \alpha_t^0 Q_{h}^{\pi}(x,a) + \sum_{i=1}^t\alpha_t^i\left(r_h(x,a) + \mathbb{P}_hV_{h+1}^{\pi}(x,a)  + V_{h+1}^{\pi}(x_{k_i,h+1}) - \hat{\mathbb{P}}_h^{k_i}V_{h+1}^{\pi}  (x,a) \right)\nonumber\\
	&= \alpha_t^0 Q_{h}^{\pi}(x,a) + \sum_{i=1}^t \alpha_t^i \left(r_h(x,a)+V_{h+1}^{\pi}(x_{k_i,h+1}) + \left\{\mathbb{P}_hV_{h+1}^{\pi}-\hat{\mathbb{P}}_h^{k_i}V_{h+1}^{\pi}\right\}(x,a) \right).\label{eq:qpi}
\end{align}
Then subtracting \eqref{eq:qpi} from \eqref{eq:qupdate} yields
\begin{align*}
	(Q_{k,h} - Q_{h}^{\pi})(x,a) =& \alpha_t^0\left\{Q_{(T-1)K^\alpha+1,h}-Q_{h}^{\pi}\right\}(x,a)\nonumber\\
	&+ \sum_{i=1}^t\alpha_t^i\left( \left\{V_{k_i,h+1}-V_{h+1}^{\pi}\right\}(x_{k_i,h+1})+\left\{\hat{\mathbb{P}}_h^{k_i}V_{h+1}^{\pi}-\mathbb{P}_hV_{h+1}^{\pi} \right\}(x,a)+b_i\right).
\end{align*}
\end{proof}

\begin{lemma}\label{le:u-hoeffding}
Consider any frame $T.$ Let t=$N_{k,h}(x,a)$ be the number of visits to $(x,a)$ at step $h$ before episode $k$ in the current frame and let $k_1,\dots,k_t < k$ be the indices of these episodes. Under any policy $\pi,$  with probability at least $1-\frac{1}{K^3},$ the following inequalities hold simultaneously for all $(x,a,h,k)\in\mathcal{S}\times\mathcal{A}\times[H]\times[K]$ 
\begin{align*}
	\left\vert \sum_{i=1}^t\alpha_t^i\left\{(\hat{\mathbb{P}}_h^{k_i}-\mathbb{P}_h) V_{h+1}^{\pi}\right\}(x,a)\right\vert \leq  &\frac{1}{4} \sqrt{\frac{H^2\iota(\chi+1)}{(\chi+t)}}, \\
		\left\vert \sum_{i=1}^t\alpha_t^i\left\{(\hat{\mathbb{P}}_h^{k_i}-\mathbb{P}_h) W_{h+1}^{\pi}\right\}(x,a)\right\vert \leq  &\frac{1}{4} \sqrt{\frac{H^2\iota(\chi+1)}{(\chi+t)}}.
\end{align*}

\end{lemma}

\begin{proof} Without loss of generality, we consider $T=1.$
	Fix any $(x,a,h)\in\mathcal{S}\times\mathcal{A}\times[H].$ 
	For any $n\in[K^\alpha],$ define $$X(n)=\sum_{i=1}^n\alpha_\tau^i\cdot \mathbb{I}_{\{k_i\leq K\}}\left\{(\hat{\mathbb{P}}_h^{k_i}-\mathbb{P}_h )V_{h+1}^\pi\right\}(x,a).$$ Let $\mathcal{F}_i$ be the $\sigma-$algebra generated by all the random variables until step $h$ in episode $k_i.$ Then 
	$$\mathbb{E}[X(n+1)\vert \mathcal{F}_n]= X(n) + \mathbb{E}\left[\alpha_\tau^{n+1}\mathbb{I}_{\{k_{n+1}\leq K\}}\left\{(\hat{\mathbb{P}}_h^{k_{n+1}}-\mathbb{P}_h )V_{h+1}^\pi\right\}(x,a) \vert \mathcal{F}_n\right]=X(n),$$
which shows that $X(n)$ is a martingale. We also have for $1\leq i \leq n,$
\begin{align*}
	\vert X(i)-X(i-1)\vert \leq  \alpha_\tau^i \left\vert \left\{(\hat{\mathbb{P}}_h^{k_{n+1}}-\mathbb{P}_h )V_{h+1}^\pi\right\}(x,a)\right\vert  \leq  \alpha_\tau^i H
\end{align*}
Then  let $\sigma = \sqrt{8\log\left(\sqrt{2SAH}K\right)\sum_{i=1}^\tau(\alpha_\tau^iH)^2}.$ By applying the  Azuma-Hoeffding inequality, we have with probability at least $1-2\exp\left(-\frac{\sigma^2}{2\sum_{i=1}^\tau(\alpha^i_\tau H )^2 }\right)=1-\frac{1}{SAHK^4},$
$$ \vert X(\tau)\vert \leq \sqrt{8\log\left(\sqrt{2SAH}K\right)\sum_{i=1}^\tau(\alpha_\tau^i H)^2}\leq  \sqrt{\frac{\iota}{16} H^2\sum_{i=1}^\tau(\alpha_\tau^i)^2}\leq \frac{1}{4}\sqrt{\frac{H^2\iota(\chi+1)}{\chi+\tau}},$$
where the last inequality holds due to $\sum_{i=1}^\tau(\alpha_\tau^i)^2\leq \frac{\chi+1}{\chi+\tau}$ from Lemma \ref{le:lr}.\ref{le:lr-e}. Because this inequality holds for any $\tau\in[K]$, it also holds for $\tau=t=N_{k,h}(x,a)\leq K,$ 
Applying the union bound, we obtain that with probability at least $1-\frac{1}{K^3}$   the following inequality holds  simultaneously for all $(x,a,h,k)\in\mathcal{S}\times\mathcal{A}\times[H]\times[K],$:
$$\left\vert \sum_{i=1}^t\alpha_t^i\left\{(\hat{\mathbb{P}}_h^{k_i}-\mathbb{P}_h) V_{h+1}^{\pi}\right\}(x,a)\right\vert \leq  \frac{1}{4} \sqrt{\frac{H^2\iota(\chi+1)}{(\chi+t)}}.$$
Following a similar analysis we also have  that with probability at least $1-\frac{1}{K^3}$   the following inequality holds  simultaneously for all $(x,a,h,k)\in\mathcal{S}\times\mathcal{A}\times[H]\times[K],$: $$\left\vert \sum_{i=1}^t\alpha_t^i\left\{(\hat{\mathbb{P}}_h^{k_i}-\mathbb{P}_h) W_{h+1}^{\pi}\right\}(x,a)\right\vert \leq \frac{1}{4} \sqrt{\frac{H^2\iota(\chi+1)}{(\chi+t)}}.$$
\end{proof}

\begin{lemma}\label{le:drift_epi_neg}
	Given $\delta\geq 2\epsilon,$ under Triple-Q, the conditional  expected drift is
	\begin{align}
	\mathbb{E}\left[L_{T+1}-L_T\vert Z_T=z \right]\leq  
	 - \frac{\delta}{2}Z_T+ \frac{4H^2\iota}{K^{2}}+ \eta \sqrt{H^2\iota}+ H^4\iota+\epsilon^2 
	\end{align}\label{le:drift_epi}
\end{lemma}

\begin{proof}
	Recall that $L_T = \frac{1}{2}  Z_T^2,$ and the virtual queue is updated by using 
	$$	Z_{T+1} = \left(  Z_T   + \rho + \epsilon -\frac{\bar{C}_T}{K^\alpha}\right)^+.$$ From inequality \eqref{eq:drift-1}, we have
	\begin{align*}
		& \mathbb{E}\left[L_{T+1}-L_T\vert Z_T=z \right] \\
		\leq &   \frac{1}{K^\alpha}\sum_{k=(T-1)K^\alpha+1}^{TK^\alpha}\mathbb E\left[ Z_T\left(\rho+\epsilon-C_{k,1}(x_{k,1},a_{k,1})  \right)  -\eta  Q_{k,1} (x_{k,1},a_{k,1}) + \eta  Q_{k,1} (x_{k,1},a_{k,1}) \vert Z_T=z \right]+H^4\iota +\epsilon^2\\
		\leq_{(a)} &  \frac{1}{K^\alpha} \sum_{k=(T-1)K^\alpha+1}^{TK^\alpha}\mathbb E\left[Z_T\left( \rho+\epsilon-\sum_a \left\{C_{k,1}q_1^{\pi}\right\}(x_{k,1},a)\right)  -\eta \sum_a \{Q_{k,1}q_1^{\pi}\} (x_{k,1},a)+ \eta  Q_{k,1} (x_{k,1},a_{k,1})\vert Z_T=z \right]\\
		&+\epsilon^2 +H^4\iota \\
		\leq &  \frac{1}{K^\alpha}\sum_{k=(T-1)K^\alpha+1}^{TK^\alpha} \mathbb E\left[  Z_T\left(\rho+\epsilon-\sum_a \left\{C_{1}^{\pi}q_1^{\pi}\right\}(x_{k,1},a) \right) -\eta \sum_a \{Q_{k,1}q_1^{\pi}\} (x_{k,1},a)+ \eta  Q_{k,1} (x_{k,1},a_{k,1})\vert Z_T=z\right]  \\
		&+\frac{1}{K^\alpha}\sum_{k=(T-1)K^\alpha+1}^{TK^\alpha} \mathbb E\left[Z_T\sum_a  \left\{C_{1}^{\pi}q_1^{\pi}\right\}(x_{k,1},a) - Z_T\sum_a \left\{C_{k,1}q_1^{\pi}\right\}(x_{k,1},a)\vert Z_T=z\right] \\
		&+\frac{1}{K^\alpha}\sum_{k=(T-1)K^\alpha+1}^{TK^\alpha}\mathbb E\left[\eta \sum_a \left\{Q_{1}^{\pi}q_1^{\pi}\right\}(x_{k,1},a)  - \eta \sum_a \left\{Q_{1}^{\pi}q_1^{\pi}\right\}(x_{k,1},a) \vert Z_T=z\right]+H^4\iota+\epsilon^2\\ 
		\leq_{(b)}  &  - \frac{\delta}{2}z+ \frac{1}{K^\alpha}\sum_{k=(T-1)K^\alpha+1}^{TK^\alpha}\mathbb E\left[\eta \sum_a\left\{(F_1^{\pi}-F_{k,1})q_{1}^{\pi}\right\}(x_{k,1},a)\ + \eta  Q_{k,1} (x_{k,1},a_{k,1})\vert Z_T=z\right]+H^4\iota+\epsilon^2 \\
		\leq_{(c)} & - \frac{\delta}{2}z+ \frac{4H^2\iota}{K^{2}}+ \eta \sqrt{H^2\iota}+ H^4\iota+\epsilon^2. 
	\end{align*}
Inequality $(a)$ holds because of our algorithm. Inequality $(b)$ holds because $\sum_a \left\{Q_{1}^{\pi}q_1^{\pi}\right\}(x_{k,1},a)$ is non-negative, and under Slater's condition, we can find policy $\pi$ such that $$\epsilon+\rho-\mathbb E\left[\sum_a {C}^{\pi}_{1}(x_{k,1},a){q}^{\pi}_{1}(x_{k,1},a) \right]=\rho+\epsilon-\mathbb E\left[\sum_{h,x,a} {q}^{\pi}_h(x,a)g_h(x,a)\right]\leq -\delta+\epsilon  \leq -\frac{\delta}{2}.$$ Finally, inequality $(c)$ is obtained similar to \eqref{eq:F-bound}, and the fact that $Q_{k,1} (x_{k,1},a_{k,1})$ is bounded by using Lemma \ref{le:q1-bound}
\end{proof}

\end{document}